\pdfoutput=1
\documentclass{article}

\usepackage[utf8]{inputenc} 
\usepackage[T1]{fontenc}    

\usepackage{fouriernc}
\usepackage{amsmath}
\usepackage[scale=0.92]{tgschola}
\usepackage[scaled=0.88]{PTSans}

\usepackage[
  paper  = letterpaper,
  left   = 1.65in,
  right  = 1.65in,
  top    = 1.0in,
  bottom = 1.0in,
  ]{geometry}

\usepackage[usenames,dvipsnames,table]{xcolor}
\definecolor{shadecolor}{gray}{0.9}

\usepackage[final,expansion=alltext]{microtype}
\usepackage[english]{babel}
\usepackage[parfill]{parskip}
\usepackage{afterpage}
\usepackage{framed}

{\endMakeFramed}

\usepackage{lineno}

\usepackage{ragged2e}

\newcounter{parcount}

\usepackage{graphicx}
\usepackage{wrapfig}
\usepackage[labelfont=bf,font=footnotesize,width=.9\textwidth]{caption}
\usepackage[format=hang]{subcaption}

\usepackage{booktabs,multirow,multicol}       

\usepackage[algoruled,algo2e]{algorithm2e}
\usepackage{listings}
\usepackage{fancyvrb}
\fvset{fontsize=\normalsize}

\usepackage[colorlinks,linktoc=all]{hyperref}
\usepackage[all]{hypcap}
\hypersetup{citecolor=MidnightBlue}
\hypersetup{linkcolor=MidnightBlue}
\hypersetup{urlcolor=MidnightBlue}

\usepackage[nameinlink, capitalize]{cleveref}

\usepackage[acronym,smallcaps,nowarn]{glossaries}[=v4.46]

\lstdefinestyle{mystyle}{
    commentstyle=\color{OliveGreen},
    keywordstyle=\color{BurntOrange},
    numberstyle=\tiny\color{black!60},
    stringstyle=\color{MidnightBlue},
    basicstyle=\ttfamily,
    breakatwhitespace=false,
    breaklines=true,
    captionpos=b,
    keepspaces=true,
    numbers=left,
    numbersep=5pt,
    showspaces=false,
    showstringspaces=false,
    showtabs=false,
    tabsize=2
}
\usepackage{amsthm}

\usepackage{centernot}
\usepackage{nicefrac}       
\usepackage{mathtools}
\usepackage{amsbsy}
\usepackage{amstext}
\usepackage{thmtools}
\usepackage{thm-restate}

\begingroup
    \makeatletter
    \@for\theoremstyle:=definition,remark,plain\do{%
        \expandafter\g@addto@macro\csname th@\theoremstyle\endcsname{%
            \addtolength\thm@preskip\parskip
            }%
        }
\endgroup

\crefname{lemma}{lemma}{lemmas}
\Crefname{lemma}{Lemma}{Lemmas}
\crefname{thm}{theorem}{theorems}
\Crefname{thm}{Theorem}{Theorems}
\crefname{prop}{proposition}{propositions}
\Crefname{prop}{Proposition}{Propositions}
\crefname{assumption}{assumption}{assumptions}
\crefname{assumption}{Assumption}{Assumptions}

\usepackage{booktabs,arydshln}
\makeatletter
\def\adl@drawiv#1#2#3{%
        \hskip.5\tabcolsep
        \xleaders#3{#2.5\@tempdimb #1{1}#2.5\@tempdimb}%
                #2\z@ plus1fil minus1fil\relax
        \hskip.5\tabcolsep}
\newcommand{\cdashlinelr}[1]{%
  \noalign{\vskip\aboverulesep
           \global\let\@dashdrawstore\adl@draw
           \global\let\adl@draw\adl@drawiv}
  \cdashline{#1}
  \noalign{\global\let\adl@draw\@dashdrawstore
           \vskip\belowrulesep}}
\makeatother

\renewcommand{\epsilon}{\varepsilon}

\declaretheorem[style=plain,name=Theorem]{theorem}
\declaretheorem[style=plain,sibling=theorem,name=Lemma]{lemma}
\declaretheorem[style=plain,sibling=theorem,name=Proposition]{proposition}

\declaretheorem[style=definition,sibling=theorem,name=Definition]{definition}
\declaretheorem[style=definition,name=Assumption]{assumption}
\declaretheorem[style=definition,sibling=theorem,name=Example]{example}
\declaretheorem[style=remark,sibling=theorem,name=Remark]{remark}

\newenvironment{example*}
 {\pushQED{\qed}\example}
 {\popQED\endexample}
\numberwithin{equation}{section}

\DeclarePairedDelimiterX\Set[1]{\lbrace}{\rbrace}%
{  #1 }

\usepackage{csquotes}
\usepackage[%
minnames=1,maxnames=99,maxcitenames=2,
style=alphabetic,
doi=false,
url=false,
giveninits=true,
hyperref,
natbib,
backend=bibtex,
sorting=nyt,
backref=true
]{biblatex}%
\renewbibmacro{in:}{%
  \ifentrytype{article}{}{\printtext{\bibstring{in}\intitlepunct}}}

\renewbibmacro*{journal}{%
  \iffieldundef{journaltitle}
    {}
    {\printtext[journaltitle]{%
       \printfield[noformat]{journaltitle}%
       \setunit{\subtitlepunct}%
       \printfield[noformat]{journalsubtitle}}}}

\DeclareFieldFormat{sentencecase}{\MakeSentenceCase*{#1}}

\renewbibmacro*{title}{%
  \ifthenelse{\iffieldundef{title}\AND\iffieldundef{subtitle}}
    {}
    {\ifthenelse{\ifentrytype{article}\OR\ifentrytype{inbook}%
      \OR\ifentrytype{incollection}\OR\ifentrytype{inproceedings}%
      \OR\ifentrytype{inreference}\OR\ifentrytype{misc}}
      {\printtext[title]{%
        \printfield[sentencecase]{title}%
        \setunit{\subtitlepunct}%
        \printfield[sentencecase]{subtitle}}}%
      {\printtext[title]{%
        \printfield[titlecase]{title}%
        \setunit{\subtitlepunct}%
        \printfield[titlecase]{subtitle}}}%
     \newunit}%
  \printfield{titleaddon}}

\AtEveryBibitem{%
\ifentrytype{article}{
    \clearfield{urldate}%
    \clearfield{eprint}
    \clearfield{eid}
}{}
\ifentrytype{book}{
    \clearfield{url}%
    \clearfield{urldate}%
    \clearfield{eprint}
}{}
\ifentrytype{collection}{
    \clearfield{url}%
    \clearfield{urldate}%
    \clearfield{eprint}
}{}
\ifentrytype{incollection}{
    \clearfield{url}%
    \clearfield{urldate}%
    \clearfield{eprint}
}{}
}

\AtEveryBibitem{
    \clearfield{pages}
    \clearfield{review}%
    \clearfield{series}
    \clearfield{volume}
    \clearfield{month}
    \clearfield{isbn}
    \clearfield{issn}
    \clearlist{location}
    \clearfield{series}
    \clearlist{publisher}
    \clearname{editor}
}{}

\usepackage{thm-restate}

\crefformat{equation}{(#2#1#3)}
\crefformat{figure}{Figure~#2#1#3}
\crefname{example}{Example}{Examples}
\crefname{lemma}{Lemma}{Lemmas}
\crefname{cor}{Corollary}{Corollaries}
\crefname{theorem}{Theorem}{Theorems}
\crefname{assumption}{Assumption}{Assumptions}

\usepackage{enumitem} 
\usepackage[separate-uncertainty=true,multi-part-units=single]{siunitx} 

\declaretheoremstyle[
spacebelow=\parsep,
    spaceabove=\parsep,
  mdframed={
    backgroundcolor=gray!10!white,     
    hidealllines=true, 
    innertopmargin=8pt, 
    innerbottommargin=4pt, 
    skipabove=8pt,
    skipbelow=10pt,
    nobreak=true
}
]{grayboxed}

\crefname{gassumption}{Assumption}{Assumptions}

\usepackage{xcolor}

\definecolor{WowColor}{rgb}{.75,0,.75}
\definecolor{SubtleColor}{rgb}{0,0,.50}

\newcounter{margincounter}

\usepackage[affil-it]{authblk}

\usepackage{makecell}

\newcommand{\ip}[2] {\ensuremath{\langle #1 , #2 \rangle}}

\newcommand{\RR}{\mathbb{R}}
\newcommand{\calD}{\mathcal{D}}

\newcommand{\calS}{\mathcal{S}}

\newcommand{\calX}{\mathcal{X}}

\newcommand{\sig}{\sigma}

\newcommand{\al}{\alpha}

\newcommand{\Del}{\Delta}

\newcommand{\calC}{\mathcal{C}}

\newcommand{\calY}{\mathcal{Y}}

\usepackage{multirow}
\usepackage{tikz}
\usetikzlibrary{positioning}
\usetikzlibrary{arrows, automata}
\usetikzlibrary{patterns}
\newcommand{\uk}{\diamond}
\newcommand{\embx}{\bar{f}}
\newcommand{\emby}{\bar{g}}
\newcommand{\latx}{h_x}
\newcommand{\laty}{h_y}
\newcommand{\latentembx}{f}
\newcommand{\latentemby}{g}

\newcommand{\core}[1]{\textrm{core}(#1)}
\newcommand{\neigh}[1]{\textrm{ne}(#1)}

\def\showauthornotes{1}
\ifnum\showauthornotes=1
\newcommand{\Authornote}[2]{{\sf\small\color{blue}{[#1: #2]}}}
\else
\newcommand{\Authornote}[2]{}
\fi

\title{On the Origins of Linear Representations\\ in Large Language Models}
\author[1]{Yibo Jiang\footnote{Equal Contribution}}
\newcommand\CoAuthorMark{\footnotemark[\arabic{footnote}]}
\author[2]{Goutham Rajendran\protect\CoAuthorMark}
\author[2]{\authorcr Pradeep Ravikumar}
\author[3]{Bryon Aragam} 
\author[4, 5]{Victor Veitch}
\affil[1]{Department of Computer Science, University of Chicago}
\affil[2]{Machine Learning Department, Carnegie Mellon University}
\affil[3]{Booth School of Business, University of Chicago}
\affil[4]{Department of Statistics, University of Chicago}
\affil[5]{Data Science Institute, University of Chicago}
\date{}

\begin{document}

\maketitle

\begin{abstract}
    Recent works have argued that high-level semantic concepts are encoded ``linearly'' in the representation space of large language models. In this work, we study the origins of such linear representations. To that end, we introduce a simple latent variable model to abstract and formalize the concept dynamics of the next token prediction. We use this formalism to show that
    the next token prediction objective (softmax with
cross-entropy) and the implicit bias of gradient descent together promote the linear representation of concepts.
Experiments show that linear representations emerge when learning from data matching the latent variable model, confirming that this simple structure already suffices to yield linear representations. We additionally confirm some predictions of the theory using the LLaMA-2 large language model, giving evidence that the simplified model yields generalizable insights.
\end{abstract}

\section{Introduction}

One of the central questions of interpretability research for language models \cite{mikolov2013linguistic, openai2023gpt4, touvron2023llama}
is to understand how high-level semantic concepts that are meaningful to humans are encoded in the representations of these models.
In this context, a surprising observation is that these concepts are often represented \emph{linearly} \citep[e.g.,][]{mikolov2013linguistic,pennington2014glove,arora2016latent, elhage2022toy, burns2022discovering, tigges2023linear, nanda2023emergent, moschella2022relative, park2023linear, li2023inference, gurnee2023finding}.
This observation is mainly empirical, leaving open major questions; most importantly: could the apparent ``linearity'' be illusory? 

In this paper, we study the origins of linear representations in large language models. 
We introduce a mathematical model for next token prediction in which context sentences and next tokens both reflect latent binary \textit{concept variables}. These latent variables give a formalization of underlying human-interpretable concepts.
Using this mathematical model, we prove that these latent concepts are indeed linearly represented in the learned representation space.
This result comes in two parts.
First, similar to earlier findings on word embeddings \citep{pennington2014glove,arora2016latent,gittens2017skip,ethayarajh2018towards}, we show log-odds matching leads to linear structure.
Second, we show that the implicit bias of gradient descent leads to the emergence of linear structure even when this (strong) log-odds matching condition fails. 
Together, these results provide strong support for the linear representation hypothesis.

There are some noteworthy implications of these results. First, linear representation structure is not specific to the choice of model architecture, but a by-product of how the model learns the conditional probabilities of different contexts and corresponding outputs. Second, the simple latent variable model gives rise to representation behavior such as linearity and orthogonality akin to those observed in LLMs. This suggests it may be a useful tool for further theoretical study in interpretability research.

The development is as follows:
\begin{enumerate}
    \item In \cref{sec: setting}, we present a latent variable model that abstracts the concept dynamics of LLM inference, allowing us to mathematically analyze LLM representations of concepts. 
    \item Using this model, we show in \cref{sec: linearity} that the next token prediction objective (softmax with cross-entropy) and the implicit bias of gradient descent together promote concepts to admit linear representations.
    \item A surprising fact about LLMs is that Euclidean geometry on representations sometimes reasonably encodes semantics, despite the Euclidean inner product being unidentified by the standard LLM objective \citep{park2023linear}. In \cref{sec: ortho}, we show that the implicit bias of gradient descent can result in Euclidean structure having a privileged status such that independent concepts are represented almost orthogonally. 
    \item Finally, in \cref{sec: expts}, we conduct experiments on simulations from the latent variable model confirming that this structure does indeed yield linear representations. Additionally, we assess the theory's predictions using LLaMA-2 \cite{touvron2023llama}, showing that the simple model yields generalizable predictions. 
\end{enumerate}

\begin{figure}[!t]
    \centering
    \hspace{-2em}
    \scalebox{.9} {
    \begin{tikzpicture}[
            > = stealth, 
            auto,
            semithick, 
            label distance=3mm,
            scale=1
        ]

        \tikzstyle{every state}=[
            draw = black,
            thick,
            fill = white,
            minimum size = 6mm,
            text width=6mm,
            align = center
        ]
        \node[state] at (0, 0) (X) {$X$};
        \node[state] at (-1.5, 3)(Z1){$C_1$};
        \node[state] at (1.5, 3)(Z2){$C_2$};
        \node[state] at (4.5, 3)(Z3){$C_3$};
        \node[state] at (3.8, 0) (Y){$Y$};
        \path[<->] (X) edge node[align = center, swap] {$\latx$ maps $x$\\
        to $c_{\core{x}}$} (0, 2.4);
        \path[<->] (3.8, 2.0) edge node[align = center] {$\laty$ maps $y$\\ to $(c_1, c_2, c_3)$} (Y);
        \draw[loosely dashed, blue] (-2.2,2.4) rectangle (2.2,3.8);
    \node[fill=white] at (1.8,4.2) {\textcolor{blue}{$\core{x}$}};
    \draw[loosely dashed, red] (-2.5,2) rectangle (5.2, 4.6);
    \node[fill=white] at (4.,5.0) {\textcolor{red}{Concept space $\calC$}};
    \path[<->,dashed] (-3, 0.3) edge node[align = center, rotate=45] {Deterministic\\ maps} (-3, 2);
    \path[->, dashed, bend left] (2.2, 3.5) edge node[align = center] {$p(.|c_{\core{x}})$} (Z3);
    \node[draw] at (0,-1.1) {The ruler of a kingdom is a};
    \node[draw] at (3.8,-1.1) {king};
    \node[] at (4.6,-1.1) {or};
    \node[draw] at (5.5,-1.1) {queen};
    \end{tikzpicture}
    }
\caption{Latent conditional model visualization. An example flow: Suppose $X$ is ``The ruler of a kingdom is a''. It first maps to a set of core concepts $C_1 = c_1, C_2 = c_2$. This leaves concept $C_3 = \uk$ unknown, which indicates say ``male'' or ``female''.
The value of $C_3 = c_3 \in \{0, 1\}$ is determined by a conditional probability $p(. | C_1 = c_1, C_2 = c_2)$ and once it's determined, the entire $c = (c_1, c_2, c_3)$ can be mapped to the next token $Y$, e.g. ``king'' (if $c_3 = 0$) or ``queen'' (if $c_3 = 1 $) respectively.}
\label{fig: setup}
\end{figure}
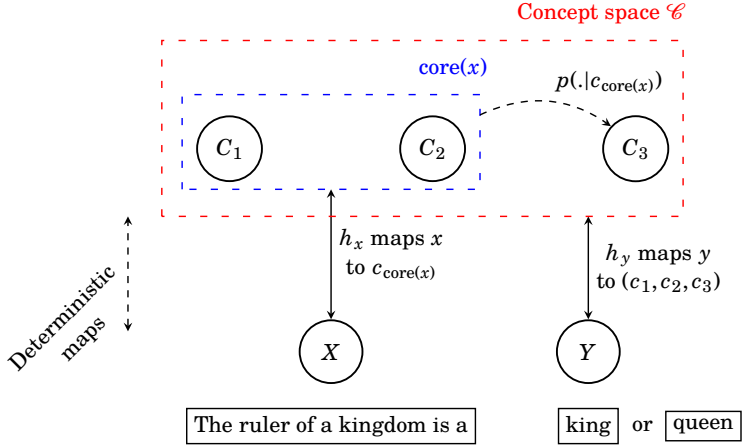

\section{Problem setting}
\label{sec: setting}

In this section, we give a simple latent variable model to study next-token predictions.

\subsection{Latent conditional model}

We let context sentences and next tokens share the same latent space where the \textit{probabilistic inference} happens. This lets us transform the next token prediction problem into the problem of learning various conditional probabilities among the latent variables.
For a visualization of the framework, see \cref{fig: setup}.

\paragraph{Latent space of concepts} We model the latent space with a set of binary random variables: $V_C = \{C_1, ..., C_m \}$.
Intuitively, one can view each latent binary random variable as representing the existence of a \emph{concept}, e.g. positive vs negative sentiment, and $V_C$ contains all relevant concepts of interest for language modeling.
Notation-wise, we will use $C = (C_1, \ldots, C_m)$ to represent the latent random vector and $c \in \{0, 1\}^m$ to denote realizations of the latent random vector. For a index set $I \subseteq [m]$,  $C_{I}$ (similarly, $c_I)$ denotes a subset of random variables. Note that the latent sample space $\mathcal{C}$ is the collection of binary vectors $\{0, 1\}^m$.

These binary random variables are not necessarily independent.
To model their dependencies, we assume that the variables form a Markov random field. That is, there is an undirected graph $G_C = (V_C, E_C)$ such that $p(C_i | C_{[m]\setminus i}) = p(C_i | C_{\neigh{i}})$ for all $i \in [m]$ (where $\neigh{i}$ denotes the set of neighbors of $i$), i.e. a variable is conditionally independent of all other variables given its neighbors. Markov random fields capture a wide variety of probability distributions, which include independent random variables as a special case.

\begin{remark}
\label{rmk: dags}
    The theory readily handles directed acyclic graphical models (DAGs) \cite{koller2009probabilistic}, but we work with undirected graphs for notational simplicity.
\end{remark}

\paragraph{Latent space to next token ($y$)} 
Let $Y \in \calY$ be the random variable denoting the next token with sample space $\mathcal{Y}$.
We assume there exists an injective measurable function from concept space $\calC$ to token space $\calY$.
That is, the value of $C$ can always be read off from $Y$.
This map induces a distribution on the next token (via the pushforward measure) that the LLM tries to learn. We denote the inverse of this map as $\laty$ (mapping tokens to concepts). 
The injectivity assumption is made for simplicity---with the injectivity assumption, for any given concept, one can identify pairs of tokens that only differ in that concept. This is useful for the theoretical analysis.

\paragraph{Context sentence ($x$) to latent space} Let $X \in \calX$ be the random variable denoting context sentence where $\mathcal{X}$ is the sample space, e.g. ``The ruler of a kingdom is a'' is an element of $\calX$. 
Intuitively, $X$ contains partial information about the latent concepts that determine the next token $Y$.
This is similar to a hidden Markov model where one can use previous observations to infer the current latent state.
In other words, if one models language generation as transitions/drifts in the latent space \cite{arora2016latent}, then given the context, one cannot fully determine but can make educated guesses on where the latent space might be transitioned into. Note that there's inherent ambiguity involving the next token prediction, e.g., multiple tokens  could follow the sentence "I am a ". 

Now, we will describe precisely how contexts $X$ and concepts $C$ relate to each other.
First, observe that
not every concept combination can induce the context sentence of interest, e.g., not every human-interpretable concept is associated with the sentence ``The ruler of a kingdom is a''.
To capture the notion of concepts relevant to a specific prompt $x$ we define:
\begin{align}
\mathcal{C}^x &= \{c \in \mathcal{C} = \{0, 1\}^m: P(x|C = c) \neq 0 \} \\
\textrm{core}(x) &= \{ i \in [m] :  (c_i = 1, \forall c \in \mathcal{C}^x) \text{ or }  ( c_i = 0, \forall c \in \mathcal{C}^x)\}.
\end{align}
Here, $\mathcal{C}^x$
is the set of all realizations of latent concept vectors in $\calC = \{0, 1\}^m$ that can induce $x$.
In other words, $\calC^x$ contains all concept vectors $c$ that are related to $x$. Within this subset, $\core{x}$ indicates the concepts that are always ``on'' or ``off''. 
See \cref{fig: setup} for an illustration.
In other words, given $x$, the values of ``core concepts'' are fixed and the rest of the non-core concepts are left to be determined.
This represents the determinacy in the relationship between $X$ and $C$.
To capture this determinancy relationship we define the map $\latx: \mathcal{X} \to  \{\uk, 0, 1\}^m$ (where $\uk$ stands for ``unknown'' and is to be read as ``unknown'') as
\begin{equation*}
    \latx(x)_i = \begin{cases}
    c_i \quad &\text{if} \; i \in \textrm{core}(x)  \; \text{and } c \in \mathcal{C}^x\\
    \uk \quad &\text{if} \; i \notin \textrm{core}(x) \\
\end{cases}
\end{equation*}
where $c_i$ is the known quantity for the core concept of index $i$. Let's denote $\mathcal{D} = \{\uk, 0, 1\}^m$, the set of all possible conditioning contexts.

\paragraph{Predictive distribution}
We make the modeling choice $p(c | x) = p(c | c_{\core{x}})$---the relationship between $x$ to $c$ factorizes via the core concepts. This captures the idea that only the core concepts present in a context sentence are relevant for the next token prediction.
So, e.g., the sentence ``The ruler of a kingdom is a '' has an induced probability of the gender concept. This induced probability is determined by the conditional prob of core concepts (e.g. $\texttt{is\_royalty} = 1$) of the sentence.
In other words, given a context sentence $x$, the joint posterior distribution of latent concepts outside of ``core concepts''---which are fixed by $x$---is determined by the latent space.
Thus, we have
\begin{equation*}
    p(y|x) = p(c|x) = p(c|c_{\textrm{core}(x)})
\end{equation*}
where $c = \laty(y)$ and the first equality follows from injectivity.

\paragraph{Notation} We summarize our notation here.
\begin{itemize}
    \item $\calX$ -- Set of all context sentences
    \item $\calY$ -- Set of all tokens
    \item $\calC = \{0, 1\}^m$ -- The set of binary concept vectors
    \item $\calD = \{\uk, 0, 1\}^m$ -- The set of context vectors with some unknown concepts
    \item $\core{x} \subseteq [m]$ -- The core concepts in $x$
    \item $\latx$ -- Deterministic map from $X$ to $\calD$
    \item $\laty$ -- Deterministic map from $Y$ to $\calC$
\end{itemize}

\subsection{Next token prediction}
The goal of the next token prediction is to learn $p(y|x)$.
Such probabilities are outputs by autoregressive large language models.
In particular, LLMs learn two functions, the embedding $\embx: \mathcal{X} \to \RR^d$ and the unembedding $\emby: \mathcal{Y} \to \RR^d$ such that
\begin{equation*}
    p(y|x) \approx \hat{p}(y|x) = \frac{\exp(\embx (x)^T \emby(y))}{\sum_y \exp(\embx (x)^T \emby(y))}
\end{equation*}

Note that under our model, $p(y|x) = p(c|c_{\textrm{core}(x)})$.
Therefore, we can assume $\embx ,\emby$ depends on the latents $c$ as well through the functions $f, g$. In particular, define $\embx = \latentembx \circ \latx$ and $\emby = \latentemby \circ \laty$, i.e. $\embx(x) = \latentembx(\latx(x))$ and $\emby(y) = \latentemby( \laty(y) )$. 

Recall that $\calC = \{0, 1\}^m$ denotes the collection of all binary concept vectors, and $\calD = \{\uk, 0, 1\}^m$ is the set encompassing all conditions. Then, equivalently, under the latent conditional model, $f$ and $g$ are trained such that their inner product can be used to estimate various conditional probabilities of the following form,
\begin{equation*}
    p(c|d) \approx \hat{p}(c|d) = \textrm{softmax}(\latentembx (d)^T \latentemby(c))
\end{equation*}
for all $c \in \widehat{\mathcal{C}}$ and $d \in \widehat{\mathcal{D}}$ where $\widehat{\mathcal{C}} \subseteq \mathcal{C}$ and $\widehat{\mathcal{D}} \subseteq \mathcal{D}$.  $\widehat{\mathcal{C}}$ and $\widehat{\mathcal{D}}$ represent the binary vectors and contexts present in the training dataset.
To prove the full slate of our results, unless otherwise stated, we assume $\widehat{\mathcal{C}} = \mathcal{C}$ and $\widehat{\mathcal{D}} = \mathcal{D}$. We will show in \cref{sec: expts} what happens when this does not hold.

In our analysis, without loss of generality, 
we represent  
elements of $\widehat{\mathcal{C}}$ and $\widehat{\mathcal{D}}$ as one-hot encodings. We let $f$ and $g$ assign unique vectors to each element (i.e., the functions are embedding lookups).

\section{Linearity}
\label{sec: linearity}

In this section, we study the phenomenon of linear representations, which we now formally define.  Adapting \citep{park2023linear}, we introduce the following definition within the latent conditional model. Recall that the cone of a vector $v$ is defined as $\mathrm{Cone}(v) = \{\alpha v : \alpha > 0\}$. For a concept $c \in \calC$ and $t \in \{0, 1\}$, let $c_{(i \to t)}$ indicate the concept $c$ with the $i$th concept set to $t$, i.e. the $j$th concept of $c_{(i \to t)}$ is $c_j$ if $j \neq i$ and $t$ otherwise. 
A similar notation can be defined for vectors in $\calD$ as well. In addition, we refer to vectors that only differ in one latent concept as counterfactual pairs (e.g. $c_{(i \to 1)}, c_{(i \to 0)}$) and the differences between their representations as steering vectors (e.g. $g(c_{(i \to 1)}) -  g(c_{(i \to 0)})$).

\begin{definition}[Linearly encoded representation]
A latent concept $C_i$ is said to have \emph{linearly encoded representation in the unembedding space} if 
there exists a unit vector $u$ such that $g(c_{(i \to 1)}) -  g(c_{(i \to 0)}) \in \mathrm{Cone}(u)$ for all $c \in \widehat{\mathcal{C}}$. Similarly, we say that $C_i$ has a \emph{linearly encoded representation in the embedding space} if there exists a unit vector $v$ such that $f(d_{(i \to 1)}) -  f(d_{(i \to 0)}) \in \mathrm{Cone}(v)$ for all $d \in \widehat{\mathcal{D}}$. In addition, 
we say $C_i$ has \emph{a matched}-representation if $u = v$.
\end{definition}

We first prove that linearly encoded representations will arise in a subspace as a result of log-odds matching (\cref{sec:linear-log-odds}) which relies on assumptions on the training data distributions. However, empirical observations (\cref{sec: expts}) reveal that, within the latent conditional model, linear representations in both the embedding and unembedding space can emerge without depending on the underlying graphical structures or the true conditional probabilities.
This leads us to establish a connection between the linearity phenomenon and the implicit bias of gradient descent on exponential loss in \cref{sec: bias_linearity}, which forms the main technical contribution of our work.

\subsection{Linearity from log-odds}
\label{sec:linear-log-odds}
In this section, we show that if the log odds of the learned probabilities is a constant that depends only on the concept of interest, then we must have linearity of representations for that concept in a subspace. 
This result is in line with many existing works on word embeddings \citep{pennington2014glove, arora2016latent, Gittens2017SkipGramZ, ethayarajh2018towards, allen2019vec, ri2023contrastive}. 
This shows that the latent conditional model is a viable theoretical framework for studying LLMs.
Going beyond prior works, our findings additionally connect linearity with the graphical structure of the latent space. An illustrative example is the case of independent concepts, discussed below. 

Concretely, the goal of this section is to argue that for any concept $i \in [m]$, the steering vectors in the unembedding space $\Del_{c, i} = g(c_{(i \to 1)}) - g(c_{(i \to 0)})$ will all be parallel for all $c \in \calC$ (up to an ambiguous subspace we cannot control).

Before we state the main theorem, we motivate the assumptions.
Firstly, we can reasonably assume the log-odds condition, which states that the model has been trained well enough to have learned correct log-odds for any concept $i$, under all contexts not conditioning on $i$. 
Secondly, since we cannot take the logarithm of $0$ conditional probabilities, it's difficult to predict how the steering vector behaves in the directions $f(d)$ for $d_i \neq \uk$.
Therefore, in this section, we project out this ambiguous subspace for now. 
However, as we will show in \cref{sec: bias_linearity}, learning these $0$ conditional probabilities with gradient descent promotes linear representations overall.

As motivated above, define $\overline{\Del_{c, i}} = \Pi_i \Del_{c, i}$ where $\Pi_i$ is the projection operator onto the space $\text{span}\{f(d) | d_i = \uk\}$, i.e., we project onto the space of contexts that does not condition on $C_i$.

\begin{restatable}[Log-odds implies linearity]{theorem}{logoddslinearity}
\label{thm: indep_linearity}
    Fix a concept $i \in [m]$. Suppose for any concept vector $c \in \calC$ and context $d \in \calD$ such that $d_i = \uk$, we have
    \begin{equation*}
    \ln \frac{\hat{p}(c_{(i \to 0)}|d)}{\hat{p}(c_{(i \to 1)}|d)} = \ln \frac{p(C_i = 0)}{p(C_i = 1)}
    \end{equation*}
    Then, the vectors $\overline{\Del_{c, i}}$ over all $c\in \calC$ are parallel.
\end{restatable}
The proof is deferred to \cref{sec: indep_linearity}.
This theorem states that for a fixed concept $C_i$, all concept steering vectors are parallel to each other, regardless of the other concepts (if we project out the ambiguous subspace). This can help to explain that in the representation space, one can have $\emby(\text{"king")} - \emby(\text{"queen"})$ (approximately) parallel to $\emby(\text{"man"}) -\emby(\text{"woman"})$ \citep{mikolov2013linguistic, park2023linear}. In this case, the two pairs $(\text{"king"}, \text{"queen"})$ and $(\text{"man"}, \text{"woman"})$ only differ in the gender concept.

\paragraph{The case of independent concepts}
We now justify that the assumption in \cref{thm: indep_linearity} holds when the concepts are jointly independent in the training distributions, therefore the theorem applies to such concepts as a special case.
Indeed, in the jointly independent case, the true distribution $p$ is a product distribution, therefore the expression on the left-hand side must match the right-hand side above and the assumption holds.

In summary, \cref{thm: indep_linearity} applies to the case of jointly independent concepts (and also more generally), implying that matching log-odds formally implies linearity of concept representations in a subspace.
In \cref{sec: mrf_linearity},
we further generalize this result to the case when the concepts are not necessarily independent and instead come from an MRF (or a DAG), where we show using log-odds that concept representations lie in a subspace of small dimension.
However, the assumption of matching log-odds may be restrictive. In order to go beyond,
we invoke ideas from the optimization literature on the implicit bias of gradient descent, which we outline in the next section.

\subsection{Linearity from implicit bias of gradient descent}
\label{sec: bias_linearity}
The goal of this section is to argue that implicit bias of gradient descent also promotes linearity in the entire space of representations. We have seen in the previous section that log-odds matching can lead to linearity. However, in the context of language modeling, it is stringent to require concepts to have matching log-odds under every possible conditioning. Furthermore, the peculiar case of linearity in the latent conditional model is that the phenomenon can occur even for randomly generated graphs and parameters of the underlying distributions (\cref{sec: expts}). This raises the question of whether there are other factors influencing the representation of concepts.

We now study the role of gradient descent in this phenomenon. However, because the next token prediction is a highly complex nonconvex optimization problem, rather than delving into the entire optimization process, our focus is on identifying \emph{subproblems} within this optimization that can result in linearly encoded representations. The goal here is to highlight the underlying dynamics that prefer linear representations. It is worth noting that we choose to study the role of gradient descent instead of the more computational tractable stochastic gradient descent used in experiments for analytic simplicity. However, according to classical stochastic approximation theory \cite{kushner2003stochastic}, with a sufficiently small learning rate, the additional stochasticity is negligible, and stochastic gradient descent will behave highly similarly to gradient descent \cite{wu2020direction}. 

The key observation is that there is a hidden binary classification task. Let's consider the simple three-variable example. Suppose instead of predicting all possible conditional distributions of latents, the model only learns $p( \cdot |c_1 = 1)$ and $p( \cdot | c_1 = 0)$. Denoting $v = g(1, 1, 1) -  g(0, 1, 1), w_0 = f(0, \uk, \uk), w_1 = f(1, \uk, \uk)$,
\begin{align*}
    \frac{\hat{p}(0, 1, 1 | c_1=1)}{\hat{p}(1, 1, 1 | c_1=1)} &= \exp(-w_1^T v) \approx 0,\\
    \frac{\hat{p}(1, 1, 1 | c_1=0)}{\hat{p}(0, 1, 1 | c_1=0)} &= \exp(w_0^T v) \approx 0
\end{align*}

Equivalent, the model is trying to optimize the following loss:
\begin{equation*}
    L(w_1, w_0, v) = \sum_{i=0}^1 \ell(-y_i w_i^T v)
\end{equation*}
where $y_0 = 1, y_1 = -1$ and $\ell(x) = \exp(-x)$. 
Intuitively, if $w_0$ and $w_1$ are fixed, then this is a binary classification task with exponential loss. It is known that gradient descent under this setting would converge to the max-margin solution \citep{soudry2018implicit}, which makes the direction of $v$ unique.

The following theorem makes this intuition precise. It states that by optimizing a specific subproblem of the next token predictions using gradient descent with \emph{embeddings fixed}, latent unembedding representations will be encoded linearly. For vectors $u,v$, denote $\cos(u, v) = \frac{\langle u, v \rangle}{||u|| \cdot ||v||}$.

\begin{restatable}[Gradient descent with fixed embeddings]{theorem}{fixedembedding}
\label{thm:fixedembedding}
Fix $i \in [m]$.
Let $\widehat{\mathcal{D}} = \{ d^{\uk}_{(i \to 1)},  d^{\uk}_{(i \to 0)}\}$ where $d^{\uk} = [\uk, ..., \uk]$ and $\Delta_{c,i} = g(c_{(i \to 1)}) -  g(c_{(i \to 0)}) $. Suppose the loss function is the following:
\begin{equation*}
\begin{split}
    L(\{\Delta_{c,i}\}_c, f(d^{\uk}_{(i \to 1)}), f(d^{\uk}_{(i \to 0)})) = 
    & \sum_{c \in \overline{\mathcal{C}}} \left(\exp(- \Delta_{c, i}^T f(d^{\uk}_{(i \to 1)})) + \exp(\Delta_{c, i}^T f(d^{\uk}_{(i \to 0)}))\right) \\
\end{split}
\end{equation*}
where for all $c \in \overline{\calC}$, $c_i = 1$ and $f(d^{\uk}_{(i \to 1)})) \neq f(d^{\uk}_{(i \to 0)})$. Then fixing $f$ and training  $g$ using gradient descent with the appropriate step size, we have
\begin{equation*}
    \lim_{t \to \infty} \cos(\Delta_{c^{1}, i}^t, \Delta_{c^{2}, i}^t) = 1
\end{equation*}
for any $c^{1}, c^{2} \in \overline{\calC}$ where the superscript $t$ is meant to represent vectors after $t$ number of iterations.
\end{restatable}

All proofs are deferred to \cref{app: bias_proof}.
The theorem says that under gradient descent with the exponential loss function and embeddings fixed, the unembedding vector differences for a fixed concept are all aligned in the limit.
It is worth mentioning that although the loss function is a bit simplified, it does represent the significant subproblem of the optimization in terms of linear geometry. The reasons are two-fold: (1) Because conditional probabilities are learned with the softmax function, there is an extra degree of freedom. In other words, one does not need to use all the logits to estimate distributions, just the differences between them. (2) Due to the nature of the exponential function, the closer the ratios between conditional distributions are to zero, the larger the distances between unembeddings need to be. Therefore, learning to predict these zero conditional probabilities would dominate the direction of concept representations.

The above theorem assumes that the embedding space is fixed. It turns out that gradient descent will also have the tendency to align embedding and unembedding space
when they're not fixed and are trained jointly, as we will show next. First, one can observe that if the $0$ conditional probabilities are learned to a reasonable approximation, then the inner product of embedding and unembedding steering vectors of the same concept should be large (\cref{prop:smallep}) or in other words, the exponential of the inner product is approaching infinity. But since there are different ways that this inner product can reach infinity, a natural question to ask is what happens when one purely optimizes this exponential learning objective. \cref{thm:align-emb-unemb} shows that gradient descent, when trained to minimize the exponential of negative inner product of two vectors, would align these two vectors.

Finally, using \cref{thm:fixedembedding} and \cref{thm:align-emb-unemb}, we obtain the following statement, which is one of our core contributions.

\begin{theorem}[Gradient descent aligns representations]
\label{thm:main}
Under the conditions of \cref{thm:fixedembedding}, if one further assumes that $\{\Delta_{c,i}\}_c, f(d^{\uk}_{(i \to 1)})), f(d^{\uk}_{(i \to 0)})$ are initialized to be mutually orthogonal with the same norm, then optimizing $f$ and $g$ using gradient descent, we have
\begin{equation*}
\begin{split}
    &\lim_{t \to \infty} \cos(\Delta_{c^{1}, i}^t, \Delta_{c^{2}, i}^t) = 1 \\
    &\lim_{t \to \infty} \cos(\Delta_{c^{1}, i}^t, f^t(d^{\uk}_{(i \to 1)}) - f^t(d^{\uk}_{(i \to 0)})) = 1 \\
\end{split}
\end{equation*}
for any $c^{1}, c^{2} \in \overline{\calC}$ where the superscript $t$ is meant to represent vectors after $t$ number of iterations.
\end{theorem}
\begin{figure*}[t]
    \centering
    \begin{subfigure}{0.3\textwidth}
        
    \centering
    \includegraphics[scale=0.23]{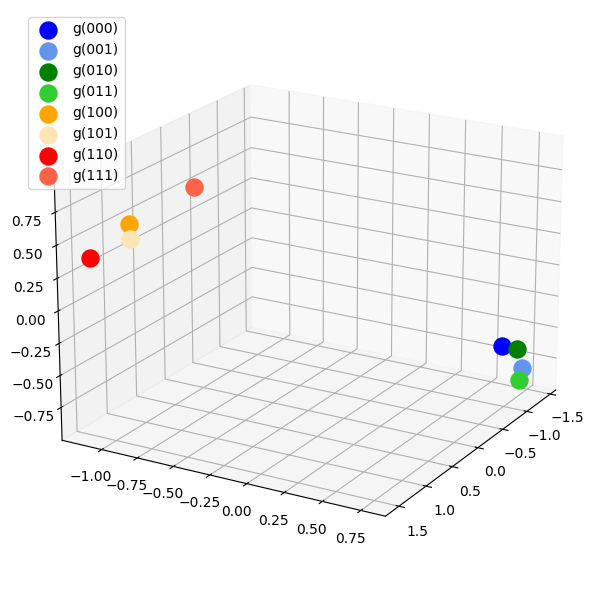}
        \caption{$\widehat{\mathcal{D}} = \{(0, \uk, \uk), (1, \uk, \uk)\}$}
        \label{fig:sub1}
    \end{subfigure}
    \hfill
    \begin{subfigure}{0.35\textwidth}
        
    \centering
    \includegraphics[scale=0.23]{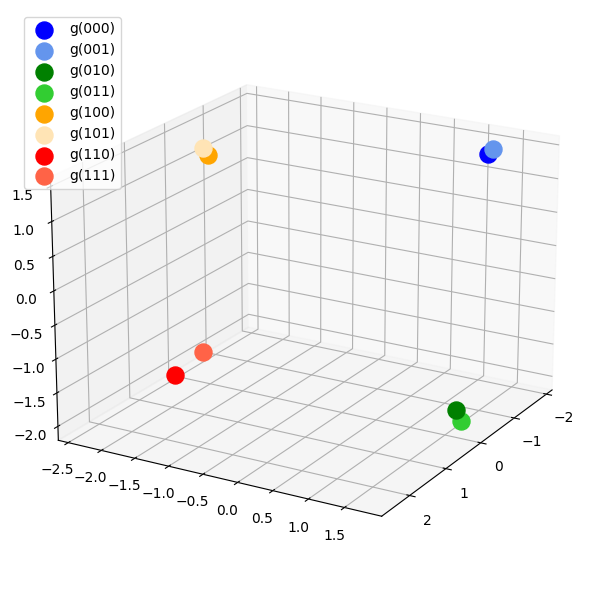}
        {\scriptsize\caption{$\widehat{\mathcal{D}} = \{(0, 0, \uk), (0, 1, \uk), (1, 0, \uk), (1, 1, \uk)\}$}}
        \label{fig:sub2}
    \end{subfigure}
    \hfill
    \begin{subfigure}{0.3\textwidth}
        
    \centering
    \includegraphics[scale=0.23]{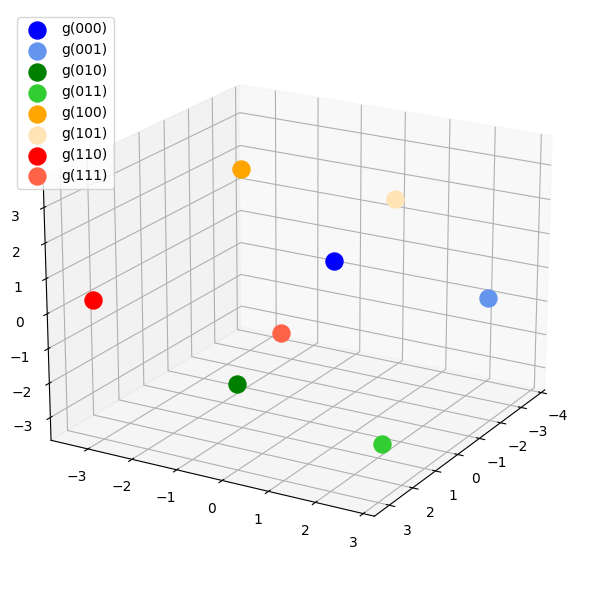}
        \caption{$\widehat{\mathcal{D}}= \mathcal{D} $}
        \label{fig:sub3}
    \end{subfigure}
    \caption{Unembedding representations form different clusters after training on various conditioning sets.
    This figure shows unembedding representations after learning different sets of conditional distributions. }
    \label{fig:3-unembedding}
\end{figure*}

While stated generally as a result of the properties of gradient descent, \cref{thm:main} says that when the latent conditional models are trained on a subproblem of next token prediction with reasonable initializations (see below remark), implicit bias of gradient descent will align the embedding and unembedding vectors as well as aligning among the unembedding vectors.
This can manifest as linear representations in the final trained LLMs such as LLaMA-2 \cite{touvron2023llama}.
Although the theorem addresses only a specific subproblem, it notably suggests, as discussed earlier, the underlying bias and dynamics that contribute to the emergence of linear representations. That is, if one solely optimizes this subproblem, one can get perfect linear representations. We verify this result empirically in \cref{sec: expts}.

\begin{remark}
The initial condition assumption in the theorem statement seems unnecessary as evidenced by our experiments in \cref{sec: expts}. On the other hand, if one initializes vectors with high dimensional multivariate Gaussians, then the assumption will be approximately satisfied with high probability \citep{dasgupta1999learning}.
\end{remark}

\paragraph{The role of learning various conditional distributions} 
So far, we have only discussed learning conditional distributions for a single concept. In general, due to the presence of zero probabilities in various conditioning contexts, learning these conditional probabilities is, in some sense, learning different forces to push sets of representations far apart because one needs relatively large inner products to get exponents to be close to zero. 

\cref{fig:3-unembedding} gives an illustrative example with three concept variables that shows how learning different subsets of conditional probabilities would push various subsets of unembeddings in different ways. For example, if the model exclusively learns conditional probabilities for $C_1$, the unembedding representations would form clusters based on the values of $C_1$. Similarly, if the model learns conditional probabilities for both $C_1$ and $C_2$, the unembedding representations would cluster based on the values of both $C_1$ and $C_2$. However, in the second case, the concept directions of $C_1$ and $C_2$ are coupled. As shown in \cref{fig:3-unembedding}(b), $g(110) - g(100)$ is parallel with $g(111) - g(101)$ but not with $g(010) - g(000)$. To decouple them, one needs a bigger $\widehat{\calD}$ as in \cref{fig:3-unembedding}(c). That said, it seems that a complete set of $\calD$ is not always necessary (\cref{sec: expts}).

\looseness=-1
An analogy would be to view each binary vector as a combination of ``electrons'' where the values represent positive or negative charges. Learning probabilities under different conditioning contexts can be likened to applying various forces to manipulate the positions of these electrons. This is similar to the Thomson Problem in chemistry and occurs in other contexts of representation learning as well \cite{elhage2022toy}.

\section{Orthogonality}
\label{sec: ortho}

A surprising phenomenon observed in LLMs like LLaMA-2 \cite{touvron2023llama} is that the Euclidean geometry can somewhat capture semantic structures, despite the fact that the Euclidean inner product is not identified by the training objective \citep{park2023linear}. Consequently, one notable outcome is the tendency for semantically unrelated concepts to be represented as almost orthogonal vectors in the unembedding space.
In this section, we study this phenomenon under our latent conditional model. Based on our underlying Markov random field distribution, we define unrelated concepts to be those separated in $G_C$.

\begin{restatable}{theorem}{thmortho}
\label{thm:ortho}
Let $\widehat{\mathcal{C}} = \mathcal{C}$ and $\widehat{\mathcal{D}} = \mathcal{D}$. Assuming $p(c) > 0$ for any $c \in \mathcal{C}$ and $C_i$ and $C_j$ are two latent variables separated in $G_C$. 
Given any binary vector $c \in \mathcal{C}$, there exists a subset $\mathcal{D}_c \subset \mathcal{D}$ such that $d_i = \uk$ and ${p}(c|d) > 0$ for any $d \in \mathcal{D}_c$. If one further assume that $\hat{p}(c|d) = {p}(c|d)$ for any $d \in \mathcal{D}_c$,
then
\begin{equation*}
    g(c_{(i \to 1)}) -  g(c_{(i \to 0)}) \perp f(d_{(j \to c_j)}) - f(d_{(j \to \uk)})
\end{equation*}
for any $d \in \mathcal{D}_c$.
\end{restatable}

All proofs are deferred to \cref{sec: ortho_proofs}.
As suggested in \cref{sec: bias_linearity} and verified empirically in \cref{sec: expts}, under the latent conditional model, representations are not only linearly encoded but also frequently aligned between embedding and unembedding spaces, which implies the following corollary on the orthogonal representation of unrelated concepts.

\begin{restatable}{cor}{corortho}
Under the conditions of \cref{thm:ortho}, if concept $C_i$ and $C_j$ have matched-representations, then
\begin{equation*}
    g(c_{(i \to 1)}) -  g(c_{(i \to 0)}) \perp g(c_{(j \to 1)}) -  g(c_{(j \to 0)})
\end{equation*}
for any $c \in \mathcal{C}$.
\end{restatable}

In simulation, given enough dimensions, the latent conditional model tends to learn near orthogonal representations regardless of graphical dependencies. 
One plausible rationale behind this phenomenon is the tendency for steering vectors in both unembedding and embedding spaces to possess large norms because, by \cref{prop:smallep}, the inner product of embedding and unembedding steering vectors of the same concept is large. However, for steering vectors representing distinct concepts, it's often necessary for the inner product to be small as dictated by what's given in the training distributions. To accommodate for this, their cosine similarities tend to be close to zero. As a consequence, unembedding representations of different concepts  often end up being near orthogonal as well.

\begin{remark}
It is worth noting that the conditions for orthogonal representations presented here can break in practice for LLMs. For instance, the injectivity assumption might not hold, and embedding and unembedding representations might not have perfectly matched representations.
\end{remark}

\section{Experiments}
\label{sec: expts}

In this section, we present two slates of experiments to validate and augment our theoretical contributions. 
Specifically, in \cref{sec: sim}, we run experiments on simulated data from the latent conditional model to verify the existence of linear and orthogonal representations in this theoretical framework and how training on smaller dimensions, incomplete sets of contexts, and concept vectors can affect the representations. Then, we run experiments on large language models in \cref{sec: llm_expts} to show nontrivial alignment between embedding and unembedding representations of the same concept as predicted by our theory. 

\subsection{Simulated experiments}
\label{sec: sim}

\paragraph{Simulation setup}
For simulation experiments, we first create simulated datasets by initially creating random DAGs (see \cref{rmk: dags}) with $m$ variables/concepts. For a specific random DAG, the conditional probabilities of one variable given its parents are modeled by Bernoulli distributions where the parameters are sampled uniformly from $[0.3, 0.7]$. 
For example, suppose the DAG is $C_1 \to C_2$. Then $C_1 \sim \mathrm{Bern}(p_1)$, $(C_2 | C_1=0) \sim \mathrm{Bern}(p_2)$, $(C_2 | C_1=1) \sim \mathrm{Bern}(p_3)$ where $p_1, p_2, p_3 \sim \mathrm{Unif}~([0.3, 0.7])$. In other words, we generate random graphical models wherein both the structures and distributions are created randomly. From these random graphical models, we can sample values of variables which are binary vectors as our datasets.

To let models learn conditional distributions, we train them to make predictions under cross-entropy loss after randomly masking the sampled binary vectors. Unless otherwise stated, the model is trained using stochastic gradient descent with a learning rate of $0.1$ and batch size $100$. More details are deferred to \cref{app: sim}.

\paragraph{Complete set of conditionals ($\widehat{\mathcal{C}} = \mathcal{C}, \widehat{\mathcal{D}} = \mathcal{D}$)} 
We first run experiments where the model is trained to learn the complete set of conditional distributions. The representation dimension is set to be the same as the number of latent variables. Given a concept, to measure linearity, we calculate the average cosine similarities among steering vectors in the unembedding space and in the embedding space as well as between steering vectors in these two spaces. The result is shown in \Cref{tab:complete-cond} which suggests that under the latent conditional model with the complete set of contexts, learned representations are indeed linearly encoded. In addition, \cref{fig:loss} in \cref{app: sim} shows that as loss decreases, the cosine similarities increase rapidly. 
Due to the exponential number of vectors that need to be trained, one cannot directly simulate a large number of latent variables. However, as we will show later (\cref{app: sim}), one can use incomplete sets $\widehat{\mathcal{C}}$ and $\widehat{\mathcal{D}}$ and simulate with more latent variables.

\begin{table}[!t]
\caption{When the model is trained to learn the complete set of conditionals, the $m$ latent variables are represented linearly, and the embedding and unembedding representations are matched. The table shows average cosine similarities among and between steering vectors of unembeddings and embeddings. Standard errors are over $100$ runs for $3$ variables and $4$ variables, $50, 20$ and $10$ runs for $5, 6$ and $7$ variables respectively.}
\label{tab:complete-cond}
\begin{center}
\begin{small}
\begin{sc}
\begin{tabular}{lcccr}
\toprule
\makecell{$m$} & Unembedding & Embedding & \makecell{Unembedding \\ and Embedding}\\
\midrule
3    & 0.972$\pm$0.006 & 0.982$\pm$0.005 & 0.980$\pm$0.005 \\
4    & 0.975$\pm$0.005 & 0.971$\pm$0.005 & 0.973$\pm$0.005\\
5    & 0.988$\pm$0.004 & 0.981$\pm$0.004 & 0.984$\pm$0.004 \\
6    & 0.997$\pm$0.000 & 0.985$\pm$0.002 & 0.990$\pm$0.001 \\
7    & 0.995$\pm$0.001 & 0.972$\pm$0.004 & 0.981$\pm$0.003 \\
\bottomrule
\end{tabular}
\end{sc}
\end{small}
\end{center}
\end{table}

\paragraph{Orthogonality}
\cref{sec: ortho} argues that, under the latent conditional model, unembedding representations of separated concepts will be orthogonal. In practice, \cref{fig:heatmap}(a) (\cref{app: sim}) shows that this happens for simulated data even without latent variables being separated in the latent graph. We test 
the orthogonalities of representations in the latent conditional model
by calculating the average cosine similarities between different sets of steering vectors.
To see how well the theoretical framework fits LLMs, we similarly calculate the average cosine similarities between different sets of steering vectors in LLaMA-2 \cite{touvron2023llama} using the counterfactual pairs in \citep{park2023linear} and they're reported in \cref{fig:llama-2}. In particular, \cref{fig:heatmap} (\cref{app: sim}) shows that both simulated experiments and LLaMA-2 experiments exhibit similar behaviors where steering vectors of the same concept are more aligned while steering vectors of different concepts are almost orthogonal.
This validates the claim that the latent conditional model is a suitable proxy to study LLMs.

In real-life datasets, not every concept or context vector has a natural correspondence to a token or sentence. On the other hand, the number of tokens and sentences grows exponentially with the number of latent concepts. Therefore, it is not realistic to always have $\widehat{\calD} = \calD$ or $\widehat{\calC} = \calC$. In other words, one does not need the next token prediction to learn all possible conditional probabilities perfectly. Because we are in a simulated setup, it's easy to probe the behavior in these situations. Therefore,
we also perform additional experiments to observe these aspects of our simulated setup, in particular (i) Training on an incomplete set of contexts ($\widehat{\mathcal{D}} \subset \mathcal{D}$), (ii) Incomplete set of concept vectors ($\widehat{\mathcal{C}} \subset \mathcal{C}$). These are deferred to \cref{app: sim}. The experiments show that linearity is robust to these changes.

Moreover, in practice, the representation dimension is typically much smaller than the number of concepts represented. Thus, we run additional experiments in \cref{app: sim} with decreasing dimensions and observe one can still get reasonable linear representations. 

Finally, we show in \cref{app: sim} that gradient descent or stochastic gradient descent is not the only algorithm that can induce linear representation, running other first-order methods like Adam \cite{kingma2014adam} can lead to a similar pattern.

\begin{figure}[!h]
    \centering
    \includegraphics[scale=0.3, trim={0em 0em 0em 0em},clip]{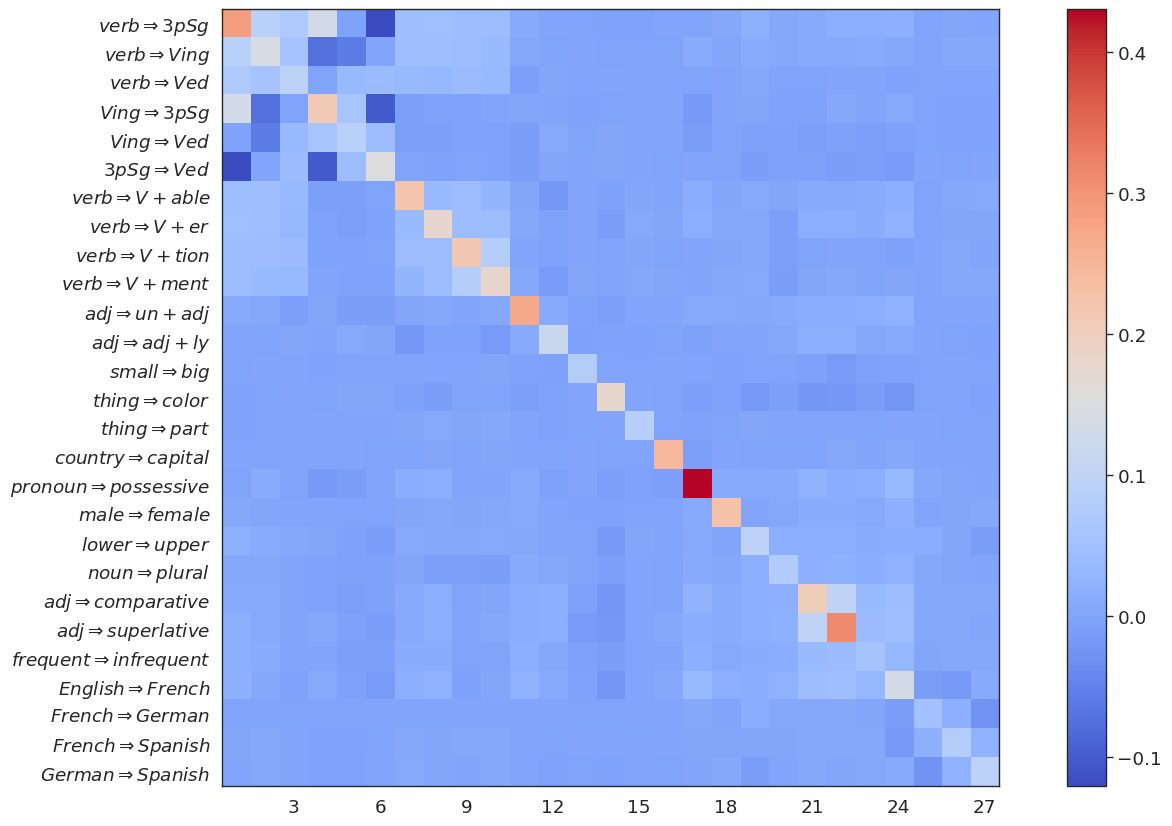}
    \caption{Unembedding steering vectors of the same concept in LLaMA-2 have nontrivial alignment, but steering vectors of different concepts are represented almost orthogonally.}
    \label{fig:llama-2}
\end{figure}

\subsection{Experiments with large language models}
\label{sec: llm_expts}

In this section, we will conduct experiments on pre-trained large language models.
We first emphasize that prior works \citep{elhage2022toy, burns2022discovering, tigges2023linear, nanda2023emergent, moschella2022relative, park2023linear, gurnee2023finding} have already exhibited certain geometries of LLM representations related to linearity and orthogonality, via experiments.
Therefore, in this section, we probe the geometry of the LLM representations, in particular that of the interplay between the embeddings and unmbeddings of the contexts and the tokens, that have not been explored in prior works.
As \cite{park2023linear} remark, for a given binary concept, it's hard to generate pairs of contexts that differ in precisely this context, partly because of nuances of natural language and mainly because such counterfactual sentences are hard to construct even for human beings. 
In this section, we try to recreate these sets of missing experiments in the literature using existing open-source datasets.

\paragraph{Multilingual embedding geometry}
We first consider language translation concepts.
For the embedding vectors, we consider pairs of contexts $(x^0, x^1)$ where $x^0, x^1$ are the same sentences but in different languages.
We consider four language pairs French--Spanish, French--German, English--French, and German--Spanish from the OPUS Books dataset \cite{tiedemann2012parallel}.
We take 150 random samples and filter out contexts with less than 20 or more than 150 tokens. 
For the unembedding concept vectors, we use the 27 concepts as described in \cite{park2023linear}, which were built on top of the Big Analogy Test dataset \cite{gladkova2016analogy}. Examples of both datasets and a list of the 27 concepts are shown in \cref{app: llm_expts}.

\begin{figure}[!h]
    \centering
    \includegraphics[scale=0.16, trim={0em 0em 0em 0em},clip]{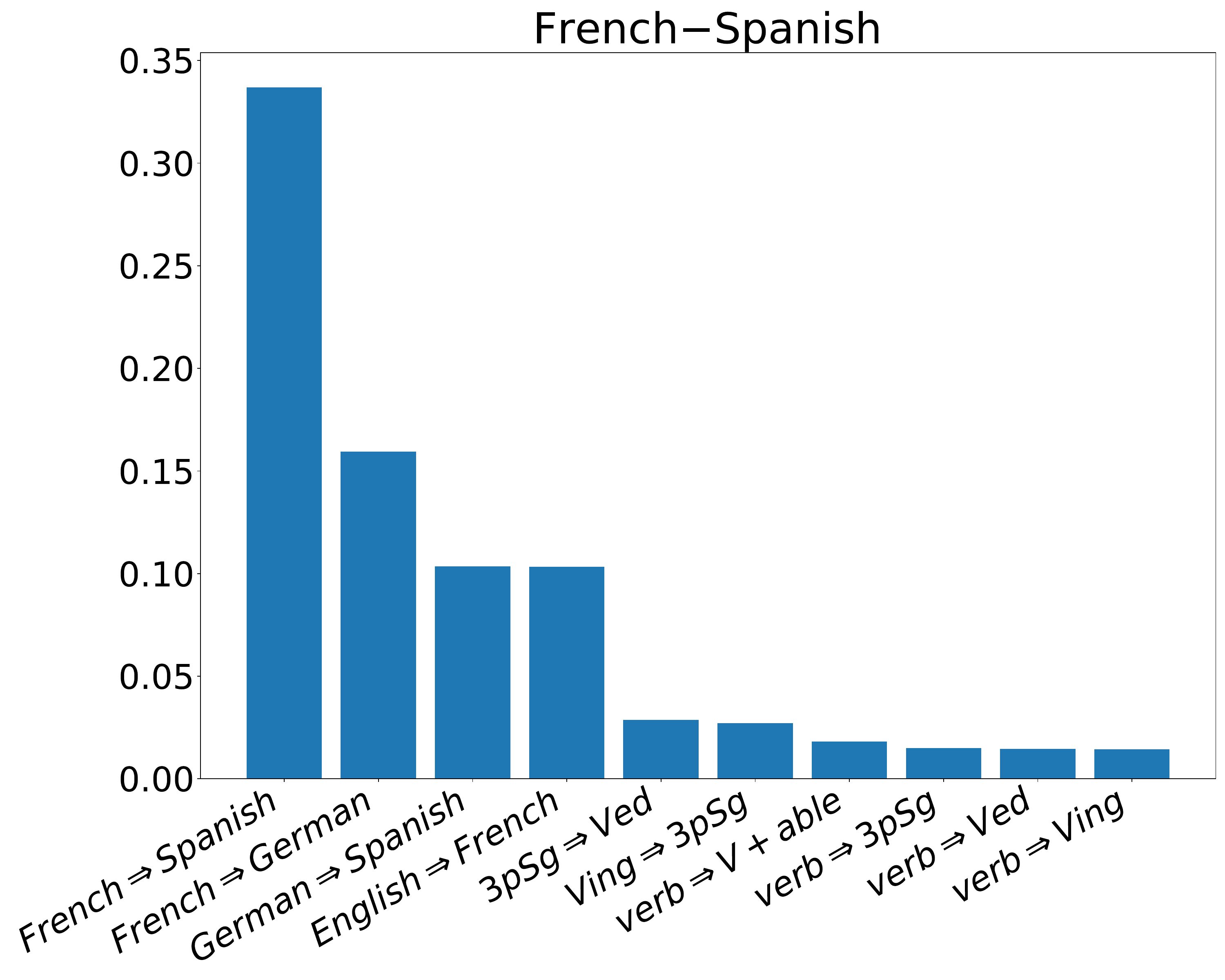}
    \caption{The French--Spanish concept is highly correlated with similar token concepts relative to others. Figure shows cosine similarities between the French--Spanish concept and token concepts.}
    \label{fig: es-fr_top10}
\end{figure}

We consider embeddings and unembeddings from LLaMA-2-7B \cite{touvron2023llama}. We compute the absolute cosine similarity between the average embedding steering vector and the average unembedding steering vector.
A barplot for French--Spanish with top 10 similarities is shown in \cref{fig: es-fr_top10}. 
The entire barplot with 27 concepts and the barplots for the other language translation concepts are in \cref{app: llm_expts}. 
As we can see, there is alignment between the embedding and unembedding representations for matching concepts relative to unmatched concepts, as \cref{thm:main} predicts.

\paragraph{Winograd Schema}
Next, we consider counterfactual context pairs arising from the
Winograd Schema dataset \cite{levesque2012winograd}, which is a dataset of pairs of sentences that differ in only one or two words and which further contain an ambiguity that can only be resolved with world knowledge and reasoning.
We run experiments similar to the multilingual embedding geometry experiments and also run additional experiments on similarities between Winograd context pairs and the 27 concepts from \cite{park2023linear}.
The experiment details are in \cref{app: llm_expts}.

\looseness=-1
These LLM experiments show that matching contexts are better aligned with the corresponding unembedding steering vectors than non-matching contexts, as predicted from our theory.
Note that although the final cosine similarities are lower than what was exhibited in the simulated experiments, this is not surprising due to the complexity and nuances of natural language and LLMs.
Another reason is that abstract concepts do not necessarily lie in a one-dimensional space, so standard cosine similarity metrics may not be an optimal choice here. However, the partial alignment already serves to strongly validate our theoretical insights.

\section{Literature review}
\label{sec: related_work}

\paragraph{Linear representations} Linear representations have been observed empirically in word embeddings \citep{mikolov2013linguistic,pennington2014glove,arora2016latent} and large language models \citep{elhage2022toy, burns2022discovering, tigges2023linear, nanda2023emergent, moschella2022relative, park2023linear, gurnee2023finding}. 
Many works in the pre-LLM era, also attempt to explain this phenomenon theoretically. 
For instance, \citet{arora2015random} and their follow-up works \citep{arora2018linear, frandsen2019understanding} study the RAND-WALK model in which latent vectors undergo continuous drift on the unit sphere. A similar notion can be found also in dynamic topic modeling \citep{blei2006dynamic} and its subsequent works \citep{rudolph2016exponential, rudolph2017dynamic}. In contrast, the latent conditional model in this paper studies discrete latent concept variables that can capture semantic meanings. 

Other works include the paraphrasing model \citep{Gittens2017SkipGramZ,allen2019analogies,allen2019vec}, which proposes to study a subset of words that are semantically equivalent to a single word, however the uniformity assumption is somewhat unrealistic. \citet{ethayarajh2018towards} explore how the linearity property can arise by decomposing pointwise mutual information matrix and \citet{ri2023contrastive} examine the phenomenon from the perspective of contrastive loss which all rely on assumptions on the matching of probability ratios. 

\Citet{wang2023concept} also use a latent variable model between prompt and output to establish the linear structure of representations. However, they study this phenomena in the context of text-to-image diffusion models, and their construction relies heavily on the Stein score representation. By contrast, the model here is closer to the standard decoder-only LLM setup, and highlights the key role of softmax.

Linearity has also been observed in other domains such as computer vision \cite{radford2015unsupervised, raghu2017svcca, bau2017network, engel2017latent, kim2018interpretability, trager2023linear, wang2023concept} and other intelligent systems \cite{mcgrath2022acquisition, schut2023bridging}. We expect our ideas to extend to other domains, however we restrict our attention to language modeling in this work.

\paragraph{Geometry of representations} There's a body of work on studying the geometry of word
and sentence representations \citep{Mimno2017TheSG, reif2019visualizing, volpi2021natural, Volpi2020EvaluatingNA, li2020sentence, chen2021probing, chang2022geometry, liang2022mind, jiang2023uncovering, park2023linear}. In particular, \citet{Mimno2017TheSG} and \citet{liang2022mind} discover representation gaps in the 
context of word embeddings and vision-language models. \citet{park2023linear} attempts to define a suitable inner product that captures semantic meanings. \citet{jiang2023uncovering} study the connection between independence and orthogonal representation by adopting the abstract notion of independence models. One can also view embeddings in the context of information geometry \citep{volpi2021natural, Volpi2020EvaluatingNA}.

\paragraph{Causal representation learning}

There is a subtle connection between our work and the field of causal representation learning that we now briefly comment on.
Causal representation learning \cite{scholkopf2021towards, scholkopf2022statistical} is an emerging field that exploits ideas from the theory of latent variable modeling \cite{arora2013practical, kivva2021learning, hyvarinen2023identifiability} along with causality \cite{spirtes2000causation, pearl2009causality, rajendran2021structure, squires2022causal} to build generative models for various domains. 
The main goal is to build the true generative models that led to the creation of a dataset.
This field has made exciting advances in recent years \cite{
khemakhem2020variational, falck2021multi, zimmermann2021contrastive, 
kivva2022identifiability, lachapelle2022disentanglement, rajendran2023interventional, varici2023score, concepts2024, jiang2023learning, buchholz2023learning, hyvarinen2023identifiability}. 
However, to study large foundation models, a purely statistical notion of building the true model may not reflect the entire story as one should also take into account the implicit bias of optimization as well (\cref{sec: bias_linearity}).
Moreover, identifying the entire underlying latent distributions may also not be necessary for certain practical purposes and it might be sufficient to represent only certain structure information in the geometry of representations, such as linearly encoded representations. Another example of this has been recently explored by \citep{jiang2023uncovering} namely independence preserving embedding which studies how independence structure can be stored in representations.

\section{Conclusion}

In this work, we presented a simple latent variable model and showed that using a standard LLM pipeline to learn the distribution results in concepts being linearly represented.
We observed that this linear structure is promoted in (at least) two ways---(i) matching log-odds, similar to prior works on word embeddings, and (ii) implicit bias of gradient descent. 
Additionally, experimental results show that---as predicted---the linear structure emerges from the simple latent variable model. We also saw that, as predicted in the simple model, LLaMA-2 representations exhibit alignment between embedding and unembedding representations.

\paragraph{Acknowledgements}
This work is supported by ONR grant N00014-23-1-2591, Open Philanthropy, NSF IIS-1956330, NIH R01GM140467, and the Robert H. Topel Faculty Research Fund at the University of Chicago Booth School of Business. We also acknowledge the support of AFRL and DARPA via FA8750-23-2-1015, ONR via N00014-23-1-2368, and NSF via IIS-1909816, IIS-1955532.

\printbibliography

\newpage
\appendix
\onecolumn

\section{Linearity from log-odds: Proof of \cref{thm: indep_linearity}}
\label{sec: indep_linearity}

In this section, we prove \cref{thm: indep_linearity} restated below for convenience.

\logoddslinearity*

\begin{proof}
Fix any $i \le m$ and consider the vector $\overline{\Del_{c, i}} = \Pi_i (g(c_{(i \to 0)}) - g(c_{(i \to 1)}))$ for any $c \in \calC$. 
Let $\widehat{\calD} = \{d \in \{\uk, 0, 1\}^m | d_i = \uk\}$ be the contexts that do not condition on $C_i$ and let $u_1, \ldots, u_k$ be an orthonormal basis of the space $\text{span}\{f(d) | d \in \widehat{\calD}\}$. Also, since $f(d)$ over $d \in \widehat{\calD}$ span this space, let
\begin{align*}
    u_j = \sum_{d \in \widehat{\calD}} \al_{j, d} f(d)
\end{align*}
for all $j = 1, 2, \ldots, k$ and scalars $\al_{j, d}$. Note that the $\al_{j, d}$ do not depend on $c$.
Now, for any $j \le k$,
\begin{align*}
    \ip{\overline{\Del_{c, i}}}{u_j} &= \ip{\Pi_i (g(c_{(i \to 0)}) - g(c_{(i \to 1)}))}{u_j}\\
    &= \ip{g(c_{(i \to 0)}) - g(c_{(i \to 1)})}{\Pi_iu_j}\\
    &= \ip{g(c_{(i \to 0)}) - g(c_{(i \to 1)})}{u_j}\\
    &= \ip{g(c_{(i \to 0)}) - g(c_{(i \to 1)})}{\sum_{d \in \calD} \al_{j, d} f(d)}\\
    &= \sum_{d \in \widehat{\calD}} \al_{j, d} \ip{g(c_{(i \to 0)}) - g(c_{(i \to 1)})}{f(d)}
\end{align*}
To this end, we will compute the inner product $\ip{g(c_{(i \to 0)}) - g(c_{(i \to 1)})}{f(d)}$ for a fixed $d \in \widehat{\calD}$. First,
\begin{align*}
\ln \frac{\hat{p}(c_{(i \to 0)}|d)}{\hat{p}(c_{(i \to 1)}|d)} = \ln \frac{p(C_i = 0)}{p(C_i = 1)}   
\end{align*}
But we have 
\begin{align*}
    \hat{p}(c_{(i \to 0)}|d) = \frac{\exp(f(d)^T g(c_{(i \to 0)}))}{\sum_{c'} \exp(f(d)^T g(c'))}, \qquad
    \hat{p}(c_{(i \to 1)}|d) = \frac{\exp(f(d)^T g(c_{(i \to 1)}))}{\sum_{c'} \exp(f(d)^T g(c'))}
\end{align*}
Since the denominator does not depend on $c$, rearranging implies
\begin{align*}
    \ip{g(c_{(i \to 0)}) - g(c_{(i \to 1)})}{f(d)} &= f(d)^T(g(c_{(i \to 1)}) - g(c_{(i \to 0)}))\\
    &= \ln \frac{\hat{p}(c_{(i \to 0)}|d)}{\hat{p}(c_{(i \to 1)}|d)}\\
    &= \ln \frac{p(C_i = 0)}{p(C_i = 1)}
\end{align*}
which only depends on $i$.
Call this expression $\al^{(i)}$ to get
\begin{align*}
    \ip{\overline{\Del_{c, i}}}{u_j} &= \sum_{d \in \widehat{\calD}} \al_{j, d} \ip{g(c_{(i \to 0)}) - g(c_{(i \to 1)})}{f(d)}\\
    &= \sum_{d \in \widehat{\calD}} \al_{j, d} \al^{(i)}
\end{align*}
Therefore,
\begin{align*}
    \overline{\Del_{c, i}} &= \sum_{j \le k} \ip{\overline{\Del_{c, i}}}{u_j} u_j\\
    &= \sum_{j \le k} \left(\sum_{d \in \widehat{\calD}} \al_{j, d} \al^{(i)}\right) u_j\\
    &= \al^{(i)} \sum_{j \le k, d \in \widehat{\calD}} \al_{j, d} u_j
\end{align*}
Note that regardless of $c$, the final expression is always parallel to the vector $v_i = \sum_{j \le k, d \in \widehat{\calD}} \al_{j, d} u_j$ which does not depend on $c$. This completes the proof.
\end{proof}

\section{Linearity from log-odds for general MRFs}
\label{sec: mrf_linearity}

In this section, we generalize the ideas from \cref{sec: indep_linearity} to concepts from a general Markov random field, which is more general than the independent case.
Since most of the technical ideas and motivations are in \cref{sec: indep_linearity}, we will go over them lightly here.
The goal is to study the structure of the steering vector $\Del_{c, i} = g(c_{(i \to 1)}) - g(c_{(i \to 0)})$. As we will see, instead of them all lying in a space of dimension $1$ as in the independent case (that's what being parallel means), they will now live in a subspace of low dimension.
For example, such a phenomenon was experimentally observed by \cite{li2023inference} for the concept of \textit{truthfulness}.

Here, concepts $C_1, \ldots, C_m$ come from a Markov random field with undirected graph $G_C = (V_C, E_C)$ with neighborhood set given by $\neigh{i}$. Therefore, they satisfy the property that 
$p(C_i | C_{[m]\setminus i}) = p(C_i | C_{\neigh{i}})$ for all $i \le m$.
Accordingly, we state the log-odds assumption to capture this general conditional independence.

\begin{assumption}
\label{as: log_odds_gen}
    For any concept $i \le m$, any concept vector $c \in \calC$ and context $d \in \{\uk, 0, 1\}^m$ such that $d_i = \uk$ and $d_j \neq \uk$ for all $j \in \neigh{i}$, we have
    \begin{equation*}
    \ln \frac{\hat{p}(c_{(i \to 0)}|d)}{\hat{p}(c_{(i \to 1)}|d)} = \ln \frac{p(C_i = 0 | C_{\neigh{i}} = d_{\neigh{i}})}{p(C_i = 1 | C_{\neigh{i}} = d_{\neigh{i}})}
\end{equation*}
\end{assumption}

As before, we assume above that the log odds condition holds when $i$ is not conditioned on, but we weaken this further to say that it only needs to hold specifically when every one of its neighbors $j$ has been conditioned on.

Analogously, we project out the space that could possibly contribute to a $0$ conditional probability. This is the space where either $d_i$  has been conditioned on or $d_j$ for some $j \in \neigh{i}$ has been conditioned on. Therefore, define $\overline{\Del_{c, i}} = \Pi_i \Del_{c, i}$ where $\Pi_i$ is the projection into the space $\text{span}\{f(d) | d_i = \uk, d_j \neq \uk \forall j \in \neigh{i}\}$
We now state our main theorem.

\begin{theorem}
\label{thm: mrf_linearity}
    Under \cref{as: log_odds_gen}, for any fixed $i \le m$, the vectors $\overline{\Del_{c, i}}$  for all $c \in \calC$ live in a subspace $\calS$ of dimension at most $2^{|\neigh{i}|}$.
\end{theorem}

This theorem says that if we assume that the concepts come from a general Markov random field, then the steering vectors live in a space of low dimension. Note that when the concepts are independent, we have $|\neigh{i}| = 0 \Longrightarrow 2^{|\neigh{i}|} = 1$ which means all the vectors $\overline{\Del_{c, i}}$ are parallel. Therefore this theorem is more general than \cref{thm: indep_linearity}.

\begin{proof}
We will proceed similarly to the proof of \cref{thm: indep_linearity}. 
Fix any $i \le m, c \in \calC$ and repeat the computation until we get that for any $j \le k$,
\begin{align*}
    \ip{\overline{\Del_{c, i}}}{u_j}
    &= \sum_{d \in \widehat{\calD}} \al_{j, d} \ip{g(c_{(i \to 0)}) - g(c_{(i \to 1)})}{f(d)}
\end{align*}
where  $\widehat{\calD} = \{d \in \{\uk, 0, 1\}^m | d_i = \uk, d_j \neq \uk \forall j \in \neigh{i}\}$.
Now, let's recompute $\ip{g(c_{(i \to 0)}) - g(c_{(i \to 1)})}{f(d)}$. By \cref{as: log_odds_gen}, we have
\begin{align*}
    \ln \frac{\hat{p}(c_{(i \to 0)}|d)}{\hat{p}(c_{(i \to 1)}|d)} = \ln \frac{p(C_i = 0 | C_{\neigh{i}} = d_{\neigh{i}})}{p(C_i = 1 | C_{\neigh{i}} = d_{\neigh{i}})}  
\end{align*}
which gives
\begin{align*}
    \ip{g(c_{(i \to 0)}) - g(c_{(i \to 1)})}{f(d)} &= \ln \frac{p(C_i = 0 | C_{\neigh{i}} = d_{\neigh{i}})}{p(C_i = 1 | C_{\neigh{i}} = d_{\neigh{i}})}
\end{align*}
which depends both on $i$ and $d_{\neigh{i}}$. For all $\sig \in \{0, 1\}^{|\neigh{i}|}$, denote $\al^{(i), \sig}$ to be the expression
\begin{align*}
    \al^{(i), \sig} &= \ln \frac{p(C_i = 0 | C_{\neigh{i}} = \sig)}{p(C_i = 1 | C_{\neigh{i}} = \sig)}
\end{align*}
Therefore,
\begin{align*}
    \overline{\Del_{c, i}}
    &= \sum_{j \le k} \left(\sum_{d \in \widehat{\calD}} \al_{j, d} \al^{(i), d_{\neigh{i}}}\right) u_j\\
    &= \sum_{\sig \in \{0, 1\}^{|\neigh{i}|}} \al^{(i), \sig} \left(\sum _{j \le k, d \in \widehat{\calD}, d_{\neigh{i}} = \sig} \al_{j, d}u_j\right)
\end{align*}
which lives in the span of the vectors $v^{(i), \sig} = \sum _{j \le k, d \in \widehat{\calD}, d_{\neigh{i}} = \sig} \al_{j, d}u_j$ regardless of $c$. The number of such vectors is $|\{0, 1\}^{|\neigh{i}|}| = 2^{|\neigh{i}|}$.
\end{proof}

\section{Linearity from the implicit bias of gradient descent}
\label{app: bias_proof}

In this section, we will prove \cref{thm:fixedembedding} and \cref{thm:main} and additional auxiliary theorems.

\fixedembedding*
\begin{proof}
First note that we can break the whole optimization problem into smaller subproblems where each subproblem only depends on one counterfactual pair. By Theorem 3 of \citep{soudry2018implicit}, all $\Delta_{c, i}$ converges to have the same direction as the hard margin SVM solution.
\end{proof}

\begin{proposition}
\label{prop:smallep}
    Suppose $p(c) > 0$ for any  $c \in \widehat{\mathcal{C}} = \mathcal{C}$ and   $|\hat{p}(c|d) - p(c|d)| < \epsilon$ for all $c \in  \widehat{\mathcal{C}}$ and $ d \in  \widehat{\mathcal{D}}$ such that $0 < \epsilon < p(c|d)$ for all $c, d$ where $p(c|d) > 0$.
    Then for any latent variable $C_i$, we have that
\begin{equation*}
    \exp( - (g(c_{(i \to 1)}) - g(c_{(i \to 0)}))^T (f(d_{(i \to 0)}) - f(d_{(i \to 1)})) ) < \frac{\epsilon^2}{\big(p(c_{c_i=0} | d_{c_i=0}) - \epsilon \big) \big(p(c_{c_i=1} | d_{c_i=1}) - \epsilon\big) }
\end{equation*}
for any $c \in  \widehat{\mathcal{C}}$ and any $d \in \widehat{\mathcal{D}}$.
\end{proposition}
\begin{proof}
For any $c \in \widehat{\mathcal{C}}$ and $d \in \widehat{\mathcal{D}}$, we have that
\begin{equation*}
\begin{split}
     \frac{p(c_{c_i=1} | d_{c_i=0})}{p(c_{c_i=0} | d_{c_i=0})} = 0, \qquad \frac{p(c_{c_i=0} | d_{c_i=1})}{p(c_{c_i=1} | d_{c_i=1})} = 0
\end{split} 
\end{equation*}
These ratios are well-defined because $p(c) > 0$ for any $c$. Therefore,
\begin{equation*}
\begin{split}
     \frac{\hat{p}(c_{c_i=1} | d_{c_i=0})}{\hat{p}(c_{c_i=0} | d_{c_i=0})}  < \frac{\epsilon}{p(c_{c_i=0} | d_{c_i=0}) - \epsilon}, \qquad \frac{\hat{p}(c_{c_i=0} | d_{c_i=1})}{\hat{p}(c_{c_i=1} | d_{c_i=1})}  < \frac{\epsilon}{p(c_{c_i=1} | d_{c_i=1}) - \epsilon}
\end{split} 
\end{equation*}
Thus,
\begin{equation*}
\begin{split}
    \exp( - (g(c_{(c_i=1)}) &- g(c_{(c_i=0)}))^T (f(d_{(c_i=0)}) - f(d_{(c_i=1)})) )  \\
    &=\frac{\hat{p}(c_{c_i=1} | d_{c_i=0})}{\hat{p}(c_{c_i=0} | d_{c_i=0})} \frac{\hat{p}(c_{c_i=0} | d_{c_i=1})}{\hat{p}(c_{c_i=1} | d_{c_i=1})}  < \frac{\epsilon^2}{\big(p(c_{c_i=0} | d_{c_i=0}) - \epsilon \big) \big(p(c_{c_i=1} | d_{c_i=1}) - \epsilon\big) }
\end{split}
\end{equation*}
\end{proof}

\begin{theorem}
\label{thm:align-emb-unemb}
Given loss function $L(u, v) = \exp(-u^T v)$, any starting point $u_0, v_0$ where $u_0 \neq - \alpha v_0$ for some $\alpha > 0$, and any step size $\eta < \frac{1}{L(u_0, v_0)}$, the gradient descent iterates will have the following properties:
\begin{enumerate}[label=(\alph*)]
    \item $\lim_{t \to \infty}L(t) = \lim_{t \to \infty}L(u_t, v_t) = 0$
    \item $\lim_{t \to \infty}||u_t|| \to \infty$ and $\lim_{t \to \infty}||v_t|| \to \infty$
    \item $\cos(u_t, v_t)$ increases monotonically with $t$
    \item $\lim_{t \to \infty} \cos(u_t, v_t) = 1$
\end{enumerate}
\end{theorem}

\begin{proof}
Before proving the theorem, let's write out a few equations. By the gradient descent algorithm, we have the following equations:
\begin{equation*}
\begin{split}
    u_{t+1} &= u_t + \eta L(t) v_t \\
    v_{t+1} &= v_t + \eta L(t) u_t 
\end{split}
\end{equation*}
Thus,
\begin{equation*}
\begin{split}
    || u_{t+1} ||^2 &= ||u_t||^2 + 2 \eta L(t) \langle u_t, v_t \rangle + \eta^2 L(t)^2||v_t||^2 \\
    || v_{t+1} ||^2 &= ||v_t||^2 + 2 \eta L(t) \langle u_t, v_t \rangle + \eta^2 L(t)^2||u_t||^2 \\
    \langle u_{t+1}, v_{t+1} \rangle &= (1 + \eta^2 L(t)^2) \langle u_t, v_t \rangle + \eta L(t) (||u_t||^2 + ||v_t||^2)
\end{split}
\end{equation*}

Now let's prove each claim one by one. 

First of all, we know that
\begin{equation*}
    \langle u_{t+1}, v_{t+1} \rangle - \langle u_t, v_t \rangle = \eta^2 L(t)^2 \langle u_t, v_t \rangle + \eta L(t) (||u_t||^2 + ||v_t||^2)
\end{equation*}
The difference is positive if $\langle u_t, v_t \rangle > 0$. To deal with the case of a negative inner product, we will use induction to prove that for any $t$, $\langle u_{t+1}, v_{t+1} \rangle - \langle u_t, v_t \rangle$ is positive and thus $L(t)$ decreases monotonically. 

\textbf{Base case ($t = 0$):}
The difference is positive if $\langle u_0, v_0 \rangle > 0$. Let's consider the case when  $\langle u_0, v_0 \rangle \leq 0$
\begin{equation*}
\begin{split}
     \langle u_{1}, v_{1} \rangle - \langle u_0, v_0 \rangle &= \eta^2 L(0)^2 \langle u_0, v_0 \rangle + \eta L(0) (||u_0||^2 + ||v_0||^2) \\
     & =  (\eta^2 L(0)^2 - 2\eta L(0)) \langle u_0, v_0 \rangle + \eta L(0) (u_0 + v_0)^2 \\
     & = \eta L(0) (\eta L(0) - 2) \langle u_0, v_0 \rangle + \eta L(0) (u_0 + v_0)^2 > 0
\end{split}
\end{equation*}
The last inequality is due to $\eta < \frac{1}{L(u_0, v_0)}$.

\textbf{Inductive step :}
Again, the difference is positive if $\langle u_t, v_t \rangle > 0$. Let's consider the case when  $\langle u_t, v_t \rangle \leq 0$
\begin{equation*}
\begin{split}
     \langle u_{t+1}, v_{t+1} \rangle - \langle u_t, v_t \rangle &= \eta^2 L(t)^2 \langle u_t, v_t \rangle + \eta L(t) (||u_t||^2 + ||v_t||^2) \\
     & =  (\eta^2 L(t)^2 - 2\eta L(t)) \langle u_t, v_t \rangle + \eta L(t) (u_t + v_t)^2 \\
     & = \eta L(t) (\eta L(t) - 2) \langle u_t, v_t \rangle + \eta L(t) (u_t + v_t)^2 > 0
\end{split}
\end{equation*}
The last inequality is due to $\eta < \frac{1}{L(u_0, v_0)}$ and the inductive hypothesis that $L(t)$ decreases monotonically. 

It is worth noting because $\langle u_t, v_t \rangle$ is monotonically increasing, there must exist a time $t_p$ such that $\langle u_{t_p}, v_{t_p} \rangle > 0$. If the opposite is true, then $\langle u_t, v_t \rangle$ must converge to a nonpositive number which is not possible because the difference between consecutive numbers in the sequence is strictly positive unless both $u_t$ and $v_t$ converges to zero. In other words, if $\langle u_t, v_t \rangle$ does not diverge, it must be a Cauchy sequence which needs both $u_t$ and $v_t$ to reach $0$. Suppose $u_t \to 0$ and $v_t \to 0$. We know that
\begin{equation}
\label{eqn:sum}
    u_{t+1} + v_{t+1} = (1+\eta L(t)) (u_{t} + v_{t})
\end{equation}
which means the only possible scenario that  $u_t \to 0$ and $v_t \to 0$ is when $u_0 + v_0 = 0$ which is excluded by the assumption on initial $u_0$ and $v_0$.

Because $L(t) > 0$, we know that $\lim_{t \to \infty}L(t)$ has a limit. Suppose the limit is some constant $c_1 \neq 0$. Then, we have 
\begin{equation*}
\begin{split}
     \langle u_{T}, v_{T} \rangle - \langle u_{t_p}, v_{t_p} \rangle &= \sum_{t=t_p}^{T-1} \big(\langle u_{t+1}, v_{t+1} \rangle - \langle u_t, v_t \rangle \big)\\
     &=  \sum_{t=t_p}^{T-1} \eta^2 L(t)^2 \langle u_t, v_t \rangle +  \sum_{t=t_p}^{T-1} \eta L(t) (||u_t||^2 + ||v_t||^2) \\
     & \geq  \sum_{t=t_p}^{T-1} \eta^2 c_1^2 \langle u_{t_p}, v_{t_p} \rangle\\
     &> \sum_{t=t_p}^{T-1} C
\end{split}
\end{equation*}
for some constant $C$. This would imply $\langle u_{T}, v_{T} \rangle \to \infty$ which contradicts that $\lim_{t \to \infty}L(t) > 0$. Therefore, $\lim_{t \to \infty}L(t) = 0$.

For the second property, we already know at least one of $||u_t||$ and $||v_t||$ will converge to infinity for the loss to converge to zero. Suppose one of them converges to a constant. Without loss of generality, let's assume $\lim_{t \to \infty}||u_t|| \to \infty$ and for all $t$, $||v_t|| \leq C_v$ for some constant $C_v$. This implies that $\lim_{t \to \infty} \frac{||v_t||}{||u_t||} \to 0$. On the other hand, let's consider the following equation of $q_t = \frac{||v_t||}{||u_t||}$ for $t \geq t_p$:
\begin{equation*}
\begin{split}
    q_{t+1}^2 = \frac{||v_{t+1}||^2}{||u_{t+1}||^2} &= \frac{||v_t||^2 + 2 \eta L(t) \langle u_t, v_t \rangle + \eta^2 L(t)^2||u_t||^2}{||u_t||^2 + 2 \eta L(t) \langle u_t, v_t \rangle + \eta^2 L(t)^2||v_t||^2} \\
\end{split}
\end{equation*}
If $q_t^2 < 1$, then 
\begin{equation*}
    \frac{2 \eta L(t) \langle u_t, v_t \rangle + \eta^2 L(t)^2||u_t||^2}{2 \eta L(t) \langle u_t, v_t \rangle + \eta^2 L(t)^2||v_t||^2} > 1\\
\end{equation*}
Therefore, if $q_t^2 < 1$, then $q_{t+1}^2 > q_t^2$. Similarly, if $q_t^2 > 1$, then $q_{t+1}^2 < q_t^2$. As a result, $q_t$ will not converge to zero, which is a contradiction.

For the third property, let's consider this equation.
\begin{equation*}
\begin{split}
    &\cos(u_{t+1}, v_{t+1})^2 = \frac{(\langle u_{t+1}, v_{t+1} \rangle )^2}{||u_{t+1}||^2 ||v_{t+1}||^2} \\
    & = \frac{(\langle u_{t+1}, v_{t+1} \rangle )^2}{(||u_t||^2 + 2 \eta L(t) \langle u_t, v_t \rangle + \eta^2 L(t)^2||v_t||^2) (||v_t||^2 + 2 \eta L(t) \langle u_t, v_t \rangle + \eta^2 L(t)^2||u_t||^2)} \\
\end{split}
\end{equation*}
where 
\begin{equation*}
\begin{split}
    (\langle u_{t+1}, v_{t+1} \rangle )^2 &= (\langle u_t, v_t \rangle)^2 + (2\eta^2 L(t)^2+ \eta^4 L(t)^4) (\langle u_t, v_t \rangle)^2 + \eta^2 L(t)^2 (||u_t||^4 + ||v_t||^4) \\
    &+ 2(1+\eta^2 L(t)^2) (\eta L(t)) \langle u_t, v_t \rangle (||u_t||^2 + ||v_t||^2) + 2\eta^2 L(t)^2||u_t||^2||v_t||^2
\end{split}
\end{equation*}
Therefore, 
\begin{equation*}
\begin{split}
    \cos(u_{t+1}, v_{t+1})^2 = \frac{(\langle u_t, v_t \rangle)^2 + \Delta X}{(||u_t||^2 + \Delta Y)(||v_t||^2 + \Delta Z)}
\end{split}
\end{equation*}
where
\begin{equation*}
\begin{split}
    \Delta X &= (2\eta^2 L(t)^2+ \eta^4 L(t)^4) (\langle u_t, v_t \rangle)^2 + \eta^2 L(t)^2 (||u_t||^4 + ||v_t||^4) \\
    &+ 2(1+\eta^2 L(t)^2) (\eta L(t)) \langle u_t, v_t \rangle (||u_t||^2 + ||v_t||^2) + 2\eta^2 L(t)^2||u_t||^2||v_t||^2\\
    \Delta Y &= 2 \eta L(t) \langle u_t, v_t \rangle + \eta^2 L(t)^2||v_t||^2 \\
    \Delta Z &= 2 \eta L(t) \langle u_t, v_t \rangle + \eta^2 L(t)^2||u_t||^2 \\
\end{split}
\end{equation*}
then,
\begin{equation*}
\begin{split}
    \Delta X - (\Delta Y)||v_t||^2 & - (\Delta Z) ||u_t||^2 - (\Delta Y)(\Delta Z)  \\
    &= (\eta^4 L(t)^4  - 2 \eta^2 L(t)^2 )(\langle u_t, v_t \rangle)^2 - ||u_t||^2||v_t||^2) \\
    & = (2 \eta^2 L(t)^2 - \eta^4 L(t)^4) ||u_t||^2||v_t||^2 (1 - \cos(u_t, v_t)^2)
\end{split}
\end{equation*}
This is zero if and only if $ \cos(u_t, v_t)^2 = 1$. Otherwise, it is strictly positive. Therefore, if $t \geq t_p$, by \cref{lemma:ratio}, $\cos(u_t, v_t)$ increases monotonically.

Before studying the case of $t < t_p$, let's first consider this difference.
\begin{equation*}
\begin{split}
    \Delta_t &= ||u_t||^2||v_t||^2 - (\langle u_v, v_t \rangle)^2 \\
     \Delta_{t+1} - \Delta_{t}& = (\eta^4 L(t)^4 - 2 \eta^2 L(t)^2) ||u_t||^2||v_t||^2 (1 - \cos(u_t, v_t)^2) < 0
\end{split}
\end{equation*}
Therefore, when $t < t_p$, because $\langle u_v, v_t \rangle$ increases but $\langle u_v, v_t \rangle < 0$, $(\langle u_v, v_t \rangle)^2$ decreases. Since $\Delta_t$ decreases, $||u_t||^2||v_t||^2$ need to decrease as well. Therefore, $cos(u_t, v_t)$ increases when  $t < t_p$.

In fact, because $\Delta_t \geq 0$, $\Delta_t$ must have a limit. In particular, this would also imply that $\Delta_t$ has an upper bound. 

Finally, we can prove the last property. Because $\cos(u_t, v_t)$ increases monotonically and $\cos(u_t, v_t) \leq 1$, it must have a limit. Note that the limit does not have to be $1$ for the loss to converge to $0$. The key observation is that by \cref{eqn:sum}, the direction of $u_t + v_t$ is always the same. In fact, it is reasonable to guess that both $u_t$ and $v_t$ will converge in that direction. Let's project $u_{t+1}$ onto the orthogonal complement of $\textrm{span}(u_t + v_t)$. And we want to show that the projected vector is bounded. Therefore, we only need to consider $t > t_p$.
\begin{equation*}    ||u_{t+1}^{\perp}||^2 = ||u_{t+1}||^2 - ||\frac{\langle u_{t+1}, u_t + v_t \rangle}{||u_t + v_t||}||^2
\end{equation*}
then,
\begin{equation*}
    ||u_{t+1}^{\perp} ||^2 = \frac{(1 - \eta L(t))^2(||u_t||^2||v_t||^2 - (\langle u_v, v_t \rangle)^2)}{||u_t + v_t||^2} \leq \frac{(1 - \eta L(t))^2(||u_t||^2||v_t||^2 - (\langle u_v, v_t \rangle)^2)}{||u_{t_p} + v_{t_p}||^2}
\end{equation*}
We already know that the numerator is bounded. Therefore, $||u_{t+1}^{\perp}||$ is also bounded. Similarly,  $||v_{t+1}^{\perp}||$ is also bounded. 

We know that both $||u_t||$ and $||v_t||$ diverge to infinity, thus $\lim_{t \to \infty} \cos(u_t, v_t + u_t) = 1$ and $\lim_{t \to \infty} \cos(v_t, v_t + u_t) = 1$. Because of the following well-known inequality,
\begin{equation*}
\begin{split}
    \cos(u_t, v_t + u_t)&\cos(v_t, v_t + u_t) - \sqrt{1 - \cos(u_t, v_t + u_t)^2} \sqrt{1 - \cos(v_t, v_t + u_t)^2} \\
    &\leq \cos(u_t, v_t) \\
    &\leq \cos(u_t, v_t + u_t)\cos(v_t, v_t + u_t)  + \sqrt{1 - \cos(u_t, v_t + u_t)^2} \sqrt{1 - \cos(v_t, v_t + u_t)^2}
\end{split}
\end{equation*}
we have $\lim_{t \to \infty} \cos(u_t, v_t) = 1$
\end{proof}

\begin{lemma}
\label{lemma:ratio}
Let $X, \Delta X, Y, \Delta Y, Z, \Delta Z > 0$ where $YZ \neq 0$ and $(Y + \Delta Y) (Z + \Delta Z) \neq 0$, if 
\begin{equation*}
\begin{split}
    \Delta X  >  (\Delta Y)Z + (\Delta Z) Y + (\Delta Y)(\Delta Z), \;\; \frac{X}{YZ} < 1
\end{split}
\end{equation*}
then,
\begin{equation*}
    \frac{X+\Delta X}{(Y + \Delta Y)(Z + \Delta Z)} > \frac{X}{YZ}
\end{equation*}
\end{lemma}
\begin{proof}
    The proof is straightforward. 
\end{proof}

Finally, we will prove our main \cref{thm:main}, which is equivalent to the following theorem with simplified notation.

\begin{theorem}
Given loss function 
\begin{equation*}
    L( u^{(1)}, u^{(2)}, \{ v^{(i)}\}_{i=1}^K ) = \sum_{i=1}^K \big(\exp(-\langle  u^{(1)}, v^{(i)} \rangle) + \exp(\langle u^{(2)}, v^{(i)} \rangle) \big)
\end{equation*}
Suppose at starting time, $\langle u^{(1)}_0, u^{(2)}_0 \rangle = 0$, $\langle u^{(1)}_0, v^{(i)}_0 \rangle = 0$, $\langle u^{(2)}_0, v^{(i)}_0 \rangle = 0$ and $\langle v^{(i)}_0, v^{(j)}_0 \rangle = 0$, $||u^{(1)}_0|| = ||u^{(2)}_0|| =  C_u$, $||v^{(i)}_0|| = C_v$ for some positive constants $C_v$, $C_u$ and all $i, j \in [K]$,
then the gradient descent iterates will have the following properties:
\begin{enumerate}[label=(\alph*)]
    \item $\lim_{t \to \infty} \cos(v_t^{(i)}, v_t^{(j)}) = 1$ for all $i, j \in [K]$
    \item $\lim_{t \to \infty} \cos(u^{(1)}_t - u^{(2)}_t, v_t^{(i)}) = 1$ for all $i \in [K]$
\end{enumerate}
\end{theorem}

\begin{proof}
Before proving the theorem, let's write out a few equations. For simplicity, let's denote $\ell^{1, i}_t = \exp(-\langle u^{(1)}_t, v_t^{(i)} \rangle)$, $\ell^{2, i}_t = \exp(\langle u_t^{(1)}, v_t^{(i)} \rangle)$, $\Delta_t = u^{(1)}_t - u^{(2)}_t$ and $\Delta_t^{i} = \ell^{1, i}_t u_t^{(1)} - \ell^{2, i}_t u_t^{(2)}$.

By the gradient descent algorithm with learning rate $\eta$,
\begin{equation*}
\begin{split}
    v^{(i)}_{t+1} &= v^{(i)}_{t} + \eta \Delta_t^{i} \\
    u^{(1)}_{t+1} &= u^{(1)}_{t} + \eta \sum_{j=1}^K \ell^{1, j}_t v_t^{(j)} \\
    u^{(2)}_{t+1} &= u^{(2)}_{t} - \eta \sum_{j=1}^K \ell^{2, j}_t v_t^{(j)}
\end{split}
\end{equation*}

Furthermore,
\begin{equation*}
\begin{split}
    \Delta_{t+1} &= u^{(1)}_{t+1} - u^{(2)}_{t+1}  = u^{(1)}_{t} - u^{(2)}_{t} + \eta \sum_{j=1}^K \ell^{1, j}_t v_t^{(j)} + \eta \sum_{j=1}^K \ell^{2, j}_t v_t^{(j)} \\
    & = \Delta_{t} + \eta \sum_{j=1}^K (\ell^{1, j}_t + \ell^{2, j}_t) v_t^{(j)} \\
    \Delta_{t+1}^{i} & = \ell^{1, i}_{t+1} u_{t+1}^{(1)} - \ell^{2, i}_{t+1} u_{t+1}^{(2)} =
    \ell^{1, j}_{t+1} u^{(1)}_{t} + \eta \ell^{1, j}_{t+1} \sum_{j=1}^K \ell^{1, j}_t v_t^{(j)} - \ell^{2, i}_{t+1} u^{(2)}_{t} + \eta \ell^{2, j}_{t+1} \sum_{j=1}^K \ell^{2, j}_t v_t^{(j)} \\
    & = \ell^{1, i}_{t+1} u^{(1)}_{t} - \ell^{2, i}_{t+1} u^{(2)}_{t} + \eta \big(\ell^{1, i}_{t+1} \sum_{j=1}^K \ell^{1, j}_t v_t^{(j)} + \ell^{2, i}_{t+1} \sum_{j=1}^K \ell^{2, j}_t v_t^{(j)}\big)
\end{split}
\end{equation*}

With our initial condition, by symmetry, one can show that for any $t \geq 1$.
\begin{enumerate}[label=(\alph*)]
    \item $\ell^{1, i}_t = \ell^{1, j}_t = \ell^{2, i}_t = \ell^{2, j}_t$ for any $i, j \in [K]$ and the loss decreases monotonically.
    \item $\Delta^i_{t} = \Delta^{j}_{t}$.
    \item $\langle v^{(i)}_{t}, \sum_{j=1}^K v_t^{(j)}\rangle = C^1_t$ for some positive constant $C^1_t$ that does not depend on index $i$. And $C^1_t$ increases monotonically with $t$.
    \item $\langle v^{(i)}_{t}, \Delta_t^{i}\rangle  = C^2_t$ for some positive constant $C^2_t$ that does not depend on index $i$.
    \item $\langle u_t^{(1)}, \Delta_t^{i}\rangle = - \langle u_t^{(2)}, \Delta_t^{i}\rangle  = C^3_t$ for some positive constant $C^3_t$ that does not depend on index $i$.
    \item $\langle u_t^{(1)}, v_t^{i}\rangle = - \langle u_t^{(2)}, v_t^{i}\rangle  = C^4_t$ for some positive constant $C^4_t$ that does not depend on index $i$.
\end{enumerate}
We'll prove this by induction. 

\paragraph{Base step ($t = 1$)}
First of all, by the initial condition, $\ell^{1, i}_t = \ell^{1, j}_t = \ell^{2, i}_t = \ell^{2, j}_t = 1$.
\begin{equation*}
\begin{split}
    \ell^{1, i}_1 &= \exp(-\langle u^{(1)}_1, v_1^{(i)} \rangle) = \exp\big(-\langle u^{(1)}_{0} + \eta \sum_{j=1}^K \ell^{1, i}_0 v_0^{(j)}, v^{(i)}_{0} + \eta \Delta_0^{i} \rangle\big) \\
    &= \exp\big(-\langle u^{(1)}_{0}, v^{(i)}_{0}\rangle - \eta \langle u^{(1)}_{0}, \Delta_0^{i}\rangle - \eta \langle v^{(i)}_{0}, \sum_{j=1}^K \ell^{1, j}_0 v_0^{(j)}\rangle - \eta^2 \langle \Delta_0^{i}, \sum_{j=1}^K \ell^{1, i}_0 v_0^{(i)}\rangle\big) \\
    &= \exp\big(-  \eta \langle u^{(1)}_{0}, \Delta_0^{i}\rangle - \eta \ell^{1, i}_0 ||v^{(i)}_{0}||^2 \big) = \exp\big(-  \eta \ell^{1, i}_0 ||u^{(1)}_{0}||^2 - \eta \ell^{1, i}_0 ||v^{(i)}_{0}||^2 \big) \\
    & = \exp\big(-  \eta \ell^{1, i}_0 C_u^2 - \eta \ell^{1, i}_0 C_v^2 \big) = \exp\big(-  \eta  C_u^2 - \eta C_v^2 \big) \\
    &= \ell^{1, j}_1
\end{split}
\end{equation*}
Similarly
\begin{equation*}
\begin{split}
    \ell^{2, i}_1 &= \exp(\langle u^{(2)}_1, v_1^{(i)} \rangle) = \exp\big(\langle u^{(2)}_{0} - \eta \sum_{j=1}^K \ell^{2, i}_0 v_0^{(j)}, v^{(i)}_{0} + \eta \Delta_0^{i} \rangle\big) \\
    & = \exp\big(\langle u^{(2)}_{0}, \Delta_0^{i} \rangle  - \eta \ell^{1, i}_0 ||v^{(i)}_{0}||^2 \big) = \exp\big(-  \eta  C_u^2 - \eta C_v^2 \big) \\
     &= \ell^{2, j}_1 = \ell^{1, i}_1 = \ell^{1, j}_1
\end{split}
\end{equation*}
For simplicity, let use $\ell_1 = \ell^{1, i}_1 = \ell^{2, i}_1$.
This also implies that $\Delta^i_{1} = \Delta^{j}_{1} = \ell_1 \Delta_1$.

On the other hand,
\begin{equation*}
\begin{split}
    \langle v^{(i)}_{1}, \sum_{j=1}^K v_1^{(j)} \rangle &= \langle v^{(i)}_{0} + \eta \Delta_0^{i}, \sum_{j=1}^K (v^{(j)}_{0} + \eta \Delta_0^{i}) \rangle \\
    &= ||v^{(i)}_{0}||^2 + \eta^2 \langle \Delta_0^{i}, \sum_{j=1}^K \Delta_0^{j} \rangle \\
    & = ||v^{(i)}_{0}||^2 + \eta^2 K ||\Delta_0^{i}||^2 = C_v^2 + 2\eta^2K C_u^2 \\
    & > \langle v^{(i)}_{0}, \sum_{j=1}^K v_0^{(j)} \rangle = ||v^{(i)}_{0}||^2 = C_v^2  
\end{split}
\end{equation*}
Notice that the constant does not depend on $i$.

\begin{equation*}
\begin{split}
\langle v^{(i)}_{1}, \Delta_1^{i}\rangle & = \ell_1 \langle v^{(i)}_{0} + \eta \Delta_0^{i}, \Delta_1 \rangle = \ell_1 \langle v_1^{(i)}, \Delta_{0} + \eta \sum_{j=1}^K (\ell^{1, j}_0 + \ell^{2, j}_0) v_0^{(j)} \rangle \\
& = \ell_1 \langle v^{(i)}_{0} + \eta \Delta_0^{i}, \Delta_{0} + 2 \eta \sum_{j=1}^K v_0^{(j)} \rangle = 2 \ell_1 \eta ||v^{(i)}_{0}||^2 + \ell_1 \eta (||u_0^{(1)}||^2 + ||u_0^{(2)}||^2) \\
& = 2 \ell_1 \eta C_v^2 + 2 \eta \ell_1 ||C_u||^2
\end{split}
\end{equation*}
Notice that the constant does not depend on $i$.
\begin{equation*}
\begin{split}
    \langle u_1^{(1)}, \Delta_1^{i}\rangle &= \ell_1 \langle u_1^{(1)}, \Delta_1\rangle = \ell_1 \langle u^{(1)}_{0} + \eta \sum_{j=1}^K \ell^{1, i}_0 v_0^{(j)}, \Delta_{0} + \eta \sum_{j=1}^K (\ell^{1, j}_0 + \ell^{2, j}_0) v_0^{(j)} \rangle \\
    & = \ell_1 \langle u^{(1)}_{0} + \eta \sum_{j=1}^K v_0^{(j)}, \Delta_{0} + 2\eta \sum_{j=1}^K v_0^{(j)} \rangle = \ell_1 ||u^{(1)}_0||^2 + 2 \ell_1 \eta^2 \sum_{j=1}^K ||v_0^{(j)}||^2\\
    & = \ell_1 C_u^2 + 2 \ell_1 \eta^2 K ||C_v||^2  \\
    & = - \langle u_1^{(2)}, \Delta_1^{i}\rangle
\end{split}
\end{equation*}
Again, the constant does not depend on $i$.

Finally,
\begin{equation*}
\begin{split}
    \langle u_1^{(1)}, v_1^{i}\rangle & = \langle u^{(1)}_{0} + \eta \sum_{i=1}^K \ell^{1, i}_0 v_0^{(i)}, v^{(i)}_{0} + \eta \Delta_0^{i}\rangle \\
    & = \langle u^{(1)}_{0} + \eta \sum_{i=1}^K v_0^{(i)}, v^{(i)}_{0} + \eta \Delta_0^{i}\rangle \\
    & = \eta ||u^{(1)}_0||^2 + \eta ||v_0^{(i)}||^2 = \eta C_u^2 + \eta C_v^2 \\
    & = - \langle u_1^{(2)}, v_1^{i}\rangle
\end{split}
\end{equation*}

\paragraph{Inductive step} By the inductive hypothesis, let $\ell_t = \ell_t^{1, i} = \ell_t^{2, i}$.
\begin{equation*}
\begin{split}
    \ell^{1, i}_{t+1} &= \exp(-\langle u^{(1)}_{t+1}, v_{t+1}^{(i)} \rangle) = \exp\big(-\langle u^{(1)}_{t} + \eta \sum_{j=1}^K \ell^{1, i}_t v_t^{(j)}, v^{(i)}_{t} + \eta \Delta_t^{i} \rangle\big) \\
    &= \exp\big(-\langle u^{(1)}_{t}, v^{(i)}_{t}\rangle - \eta \langle u^{(1)}_{t}, \Delta_t^{i}\rangle - \eta \langle v^{(i)}_{t}, \sum_{j=1}^K \ell^{1, j}_t v_t^{(j)}\rangle - \eta^2 \langle \Delta_t^{i}, \sum_{j=1}^K \ell^{1, i}_t v_t^{(i)}\rangle\big) \\
    &= \exp\big(-C_t^4 - \eta C_t^3 - \eta \ell_t C_t^1 - \eta^2 \ell_t K C_t^2 \big) \\
    & = \ell^{1, j}_{t+1} = \ell^{2, i}_{t+1} = \ell^{2, j}_{t+1}
\end{split}
\end{equation*}

\begin{equation*}
\begin{split}
    \langle v^{(i)}_{t+1}, \sum_{j=1}^K v_{t+1}^{(j)} \rangle &= \langle v^{(i)}_{t} + \eta \Delta_t^{i}, \sum_{j=1}^K (v^{(j)}_{t} + \eta \Delta_t^{i}) \rangle \\
    &= \langle v^{(i)}_{t}, \sum_{j=1}^K v^{(j)}_{t} \rangle + \eta \langle v^{(i)}_{t}, \sum_{j=1}^K  \Delta_t^{i}\rangle + \eta \langle \Delta_t^{i}, \sum_{j=1}^K v^{(j)}_{t}\rangle + \eta^2 \langle \Delta_t^{i}, \sum_{j=1}^K  \Delta_t^{i}\rangle \\
    & = C_t^1 + 2 \eta K C_2^t + \eta^2 \ell_t \langle u_t^{(1)} - u_t^{(2)}, \sum_{j=1}^K  \Delta_t^{i}\rangle\\
    & = C_t^1 + 2 \eta K C_2^t + \eta^2 \ell_t \langle u_t^{(1)} - u_t^{(2)}, \sum_{j=1}^K  \Delta_t^{i}\rangle\\
    & =  C_t^1 + 2 \eta K C_2^t + 2 \eta^2 \ell_t K C_t^4 > \langle v^{(i)}_{t}, \sum_{j=1}^K v^{(j)}_{t} \rangle > 0
\end{split}
\end{equation*}

By same logic, one show the inductive step for $\langle v^{(i)}_{t}, \Delta_t^{i}\rangle$, 
$\langle u_t^{(1)}, \Delta_t^{i}\rangle$, $\langle u_t^{(2)}, \Delta_t^{i}\rangle$, $\langle u_t^{(1)}, v_t^{i}\rangle$ and $\langle u_t^{(2)}, v_t^{i}\rangle$.

Therefore, we can simplify the notation. The gradient descent iterates can be rewritten as:
\begin{equation*}
\begin{split}
    v^{(i)}_{t+1} &= v^{(i)}_{t} + \eta \ell_t \Delta_t \quad u^{(1)}_{t+1} = u^{(1)}_{t} + \eta \ell_t \sum_{i=1}^K v_t^{(i)} \\
    u^{(2)}_{t+1} &= u^{(2)}_{t} - \eta \ell_t \sum_{i=1}^K v_t^{(i)} \quad \Delta_{t+1} = \Delta_{t} + 2 \eta \ell_t \sum_{j=1}^K v_t^{(j)} \\
\end{split}
\end{equation*}
Furthermore, let $V_t = \sum_{j=1}^K v^{(j)}_{t}$
\begin{equation*}
\begin{split}
    V_{t+1} &= V_t + K \eta \ell_t \Delta_t \\
    \Delta_{t+1} &= \Delta_{t} + 2 \eta \ell_t V_t
\end{split}
\end{equation*}
In essence, the rest of the proof follows from the proof of \cref{thm:align-emb-unemb} and symmetry. 

\paragraph{Loss converging to zero} One first notice that by the induction argument above, the loss must be monotonically decreasing. In fact, $\lim_{t\to \infty}\ell_t  = 0$. To see this, notice that
\begin{equation*}
    \langle u^{(1)}_{t+1}, v^{(i)}_{t+1} \rangle - \langle u^{(1)}_{t}, v^{(i)}_{t} \rangle \geq \eta \langle v_t^{(j)}, \ell_t \sum_{j=1}^K v_t^{(j)} \rangle \geq \ell_t \eta \langle v_1^{(j)}, \sum_{j=1}^K v_t^{(j)} \rangle = \ell_t \eta C_1^1
\end{equation*}
Suppose $\ell_t $ is not converging to zero. Then $\ell_t$ has an lower bound, which means $\langle u^{(1)}_{t+1}, v^{(i)}_{t+1} \rangle$ will increase to infinity. This is a contradiction. Therefore, $\lim_{t\to \infty}\ell_t = 0$.

In fact, one can also show that, $\lim_{t \to \infty}||V_t|| \to \infty$ and $\lim_{t \to \infty}||\Delta_t|| \to \infty$. To see this, one first notice that,
\begin{equation*}
    \lim_{t \to \infty} \exp(- \langle \Delta_t, V_t \rangle) = \lim_{t \to \infty} \ell_t^{2K} \to 0
\end{equation*}
Therefore, at least one of $||V_t||$ or $||\Delta_t|| $ needs to go to infinity. Suppose only $||\Delta_t||$ reaches infinity, then $\lim_{t \to \infty} \frac{||V_t||}{||\Delta_t||} \to 0$. On the other hand, 

\begin{equation*}
    \frac{||V_{t+1}||^2}{||\Delta_{t+1}||^2} = \frac{||V_t||^2 + 2K \eta \ell_t \langle V_t, \Delta_t \rangle + K^2 \eta^2 \ell^2 ||\Delta_t||^2}{||\Delta_t||^2 + 4 \eta \ell_t \langle V_t, \Delta_t \rangle + 4 \eta^2 \ell^2 ||V_t||^2}
\end{equation*}

Suppose $\frac{||V_t||^2}{||\Delta_t||^2} < K/2$, then 
\begin{equation*}
    \frac{2K \eta \ell_t \langle V_t, \Delta_t \rangle + K^2 \eta^2 \ell^2 ||\Delta_t||^2}{4 \eta \ell_t \langle V_t, \Delta_t \rangle + 4 \eta^2 \ell^2 ||V_t||^2} \geq K/2
\end{equation*}
Therefore, if $\frac{||V_t||^2}{||\Delta_t||^2} < K/2$, $\frac{||V_{t+1}||^2}{||\Delta_{t+1}||^2} > \frac{||V_t||^2}{||\Delta_t||^2}$. Similarly, if $\frac{||V_t||^2}{||\Delta_t||^2} > K/2$, $\frac{||V_{t+1}||^2}{||\Delta_{t+1}||^2} < \frac{||V_t||^2}{||\Delta_t||^2}$. Therefore, $\lim_{t \to \infty} \frac{||V_t||}{||\Delta_t||}$ will not be zero. So both $\lim_{t \to \infty}||V_t|| \to \infty$ and $\lim_{t \to \infty}||\Delta_t|| \to \infty$.

\paragraph{Cosine similarity between $v^{(i)}_t, v^{(j)}_t$ converges to one}
Note for any $i$, $\lim_{t \to \infty}||v^{(i)}_{t}|| \to \infty$. This is because by symmetry, 
\begin{equation*}
    \lim_{t \to \infty}||v^{(i)}_{t}|| \geq \lim_{t \to \infty} \frac{||V_{t}||}{K} \to \infty
\end{equation*}

Then,
\begin{equation*}
    v^{(i)}_{T} = v^{(i)}_{0} + \eta \sum_{t=0}^{T-1} \Delta^{w}_t
\end{equation*}
Let's denote $D_T = \eta \sum_{t=0}^{T-1} \Delta^{w}_t$. Then $||v^{(i)}_{0}|| + ||D_T|| \geq ||v^{(i)}_{T}||$. Thus, $\lim_{T \to \infty} ||D_T|| = \infty$.

Finally
\begin{equation*}
\begin{split}
     \cos(v^{(i)}_{t}, v^{(i)}_{t}) &= \frac{\langle v^{(i)}_{0} + D_T, v^{(j)}_{0} + D_T \rangle}{||v^{(i)}_{t}|| ||v^{(j)}_{t}||} \geq \frac{\langle v^{(i)}_{0} + D_T, v^{(j)}_{0} + D_T \rangle}{(||v^{(i)}_{0}|| + ||D_t||) (||v^{(j)}_{0}|| + ||D_t||)} \\
     &= \frac{\langle v^{(i)}_{0} + v^{(j)}_{0}, D_T\rangle + ||D_T||^2}{(||v^{(i)}_{0}|| + ||D_t||) (||v^{(j)}_{0}|| + ||D_t||)} = \frac{\langle v^{(i)}_{0} + v^{(j)}_{0}, D_T\rangle / ||D_t||^2  + 1}{(||v^{(i)}_{0}||/||D_t|| + 1) (||v^{(j)}_{0}||/||D_t|| + 1)}
\end{split}
\end{equation*}
Thus,
\begin{equation*}
    \lim_{t \to \infty}\cos(v^{(i)}_{t}, v^{(i)}_{t}) = 1
\end{equation*}

\paragraph{Cosine similarity between $v^{(i)}_t, \Delta_t$ converges to one}
Finally, one can show that $\lim_{t \to \infty} \cos(\Delta_t, V_t) \to 1$ follows the same proof of \cref{thm:align-emb-unemb}. For completeness, we will present the full proof here.

Let's first do a simple variable change,
\begin{equation*}
\begin{split}
    \sqrt{2}V_{t+1} &= \sqrt{2}V_t + \sqrt{2K} \eta \ell_t \sqrt{K}\Delta_t \\
    \sqrt{K}\Delta_{t+1} &= \sqrt{K}\Delta_{t} + \sqrt{2K}\eta \ell_t \sqrt{2}V_t
\end{split}
\end{equation*}

Let $\Tilde{V}_t = \sqrt{2} V_t$, $\Tilde{\Delta}_t = \sqrt{K} \Delta_t$, and $\Tilde{\eta} = \sqrt{2K} \eta$, then

\begin{equation*}
\begin{split}
    \Tilde{V}_{t+1} &= \Tilde{V}_t + \Tilde{\eta} \ell_t \Tilde{\Delta}_t \\
    \Tilde{\Delta}_{t+1} &= \Tilde{\Delta}_{t} + \Tilde{\eta} \ell_t \Tilde{V}_t
\end{split}
\end{equation*}

One first notice that, $\Tilde{V}_{t} + \Tilde{\Delta}_{t}$ always has the direction at any $t$. Therefore, let's consider the $\Tilde{V}_{t+1}^{\perp}$ which is the residual after projecting onto the direction of $\Tilde{V}_{t} + \Tilde{\Delta}_{t}$,
\begin{equation*}
\begin{split}
    ||\Tilde{V}_{t+1}^{\perp}||^2 &= ||\Tilde{V}_{t+1}||^2 - ||\frac{\langle \Tilde{V}_{t+1}, \Tilde{V}_{t} + \Tilde{\Delta}_{t}\rangle}{||\Tilde{V}_{t} + \Tilde{\Delta}_{t}||}||^2 \\
    &= \frac{(1-\Tilde{\eta}\ell_t)^2 \big( ||\Tilde{V}_{t}||^2||\Tilde{\Delta}_{t}||^2 - (\langle \Tilde{V}_{t}, \Tilde{\Delta}_{t} \rangle)^2 \big)}{||\Tilde{V}_{t} + \Tilde{\Delta}_{t}||^2} \\
    & \leq C_{\eta}\frac{||\Tilde{V}_{t}||^2||\Tilde{\Delta}_{t}||^2 - (\langle \Tilde{V}_{t}, \Tilde{\Delta}_{t} \rangle)^2 }{||\Tilde{V}_{t} + \Tilde{\Delta}_{t}||^2} \\
\end{split}
\end{equation*}
Note that $\Tilde{\eta} \ell_t$ converges to zero. Therefore, there's an upper bound $C_{\eta}$ on $(1-\Tilde{\eta}\ell_t)^2$.

On the other hand, let $O_t = ||\Tilde{V}_{t}||^2||\Tilde{\Delta}_{t}||^2 - (\langle \Tilde{V}_{t}, \Tilde{\Delta}_{t} \rangle)^2$. Then
\begin{equation*}
    O_{t+1} - O_{t} = (\Tilde{\eta}^4 \ell_t^4 - 2 \Tilde{\eta}^2 \ell_t^2)||\Tilde{V}_{t}||^2||\Tilde{\Delta}_{t}||^2 (1 - \cos(\Tilde{V}_{t}, \Tilde{\Delta}_{t}))
\end{equation*}
Once again, because $\ell_t$ is eventually converging to zero, $O_{t}$ will decrease at some point. This is because $\Tilde{\eta}^4 \ell_t^4 - 2 \Tilde{\eta}^2 \ell_t^2 < 0$ if $\ell_t < \frac{\sqrt{2}}{\Tilde{\eta}}$ and $||\Tilde{V}_{t}||^2||\Tilde{\Delta}_{t}||^2 (1 - \cos(\Tilde{V}_{t}, \Tilde{\Delta}_{t})) \geq 0$.
Because $O_t \geq 0$, it will reach a limit. Therefore, $O_t$ must have an upper bound.
Finally, the denominator is diverging and by our inductive statements, it must have a nonzero lower bound. 

Therefore $||\Tilde{V}_{t+1}^{\perp}||$ is bounded. And as $\lim_{t \to \infty}||V_t|| \to \infty$ and $\lim_{t \to \infty}||\Delta_t|| \to \infty$, we have 
\begin{equation*}
\begin{split}
    \lim_{t\to\infty}\cos(\Tilde{V}_{t}, \Tilde{V}_{t} + \Tilde{\Delta}_t) &= 1 \\
    \lim_{t\to\infty}\cos(\Tilde{\Delta}_{t}, \Tilde{V}_{t} + \Tilde{\Delta}_t) &= 1 \\
\end{split}
\end{equation*}

Thus, 
\begin{equation*}
    \lim_{t\to\infty}\cos(\Tilde{\Delta}_{t}, \Tilde{V}_{t}) = 1
\end{equation*}

This would also imply that $\lim_{t \to \infty} \cos(\Delta_t, v^{(i)}_t) \to 1$ for all $i$.

\end{proof}

\section{Orthogonality}

\label{sec: ortho_proofs}

In this section, we will prove our main theorems on orthogonality.

\thmortho*
\begin{proof}
First of all, for any binary vector $c$, $\mathcal{D}_c$ is non-empty by the positivity assumption because $d^{\uk} \in \mathcal{D}_c$ where $d^{\uk} = \{\uk, ..., \uk \}$.

Consider an arbitrary $c \in \mathcal{C}$ and an arbitrary $d \in \mathcal{D}_c$. Without loss of generality, let $c_i = 1$. Suppose $d_j = \uk$, then it must be that $d_{(j \to c_j)} \in \mathcal{D}_c$ because $d_{(j \to c_j)}$ agrees with $c$ on the $j$-th entry. Similarly, if $d_j = c_j$, then $d_{(j \to \uk)} \in \mathcal{D}_c$.

By the positivity assumption and the fact that $d_i = \uk$, 
\begin{equation*}
\begin{split}
    p(c_{(i \to 1)} | d_{(j \to \uk)}) > 0, &\qquad p(c_{(i \to 0)} | d_{(j \to \uk)}) > 0 \\
    p(c_{(i \to 1)} | d_{(j \to c_j)}) > 0, &\qquad p(c_{(i \to 0)} | d_{(j \to c_j)}) > 0 
\end{split}
\end{equation*}
Thus,
\begin{equation*}
\begin{split}
    \hat{p}(c_{(i \to 1)} | d_{(j \to \uk)}) = p(c_{(i \to 1)} | d_{(j \to \uk)}) &\quad \hat{p}(c_{(i \to 0)} | d_{(j \to \uk)}) = p(c_{(i \to 0)} | d_{(j \to \uk)}) \\
     \hat{p}(c_{(i \to 1)} | d_{(j \to c_j)}) = p(c_{(i \to 1)} | d_{(j \to c_j)}) &\quad \hat{p}(c_{(i \to 0)} | d_{(j \to c_j)}) = p(c_{(i \to 0)} | d_{(j \to c_j)}) \\
\end{split}
\end{equation*}

Then by the Hammersley–Clifford theorem, we can factorize the joint distribution over cliques:
\begin{equation*}
    p(c) \propto \prod_k \Psi_k (c_{I_k})
\end{equation*}
where $\Psi_k (c_{I_k})$ is a function that only depends on the clique of random variables $C_{I_k}$.

By this factorization, if $p(c_{(i \to 1)} | d) > 0$ and $p(c_{(i \to 0)} | d) > 0$,
\begin{equation*}
    \ln \frac{p(c_{(i \to 1)} | d)}{p(c_{(i \to 0)} | d)} =  \beta (c_{I_{i}}, d_{I_{i}})
\end{equation*}
where $\beta$ is some function that only depends on cliques that involve $C_i$. In other words, $i \in I_{i}$ and $i' \in I_{i}$ if $C_{i'}$ and $C_{i}$ are in the same clique in $G_C$.

Thus,
\begin{equation*}
    \ln \frac{p(c_{(i \to 1)} | d_{(j \to \uk)})}{p(c_{(i \to 0)} | d_{(j \to \uk)})} = \ln \frac{p(c_{(i \to 1)} | d_{(j \to c_j)})}{p(c_{(i \to 0)} | d_{(j \to c_j)})}
\end{equation*}
and,
\begin{equation*}
    \ln \frac{\hat{p}(c_{(i \to 1)} | d_{(j \to \uk)})}{\hat{p}(c_{(i \to 0)} | d_{(j \to \uk)})} = \ln \frac{\hat{p}(c_{(i \to 1)} | d_{(j \to c_j)})}{\hat{p}(c_{(i \to 0)} | d_{(j \to c_j)})}
\end{equation*}

Therefore,
\begin{equation*}
\begin{split}
    \bigg( g(c_{(i \to 1)}) -  g(c_{(i \to 0)}) \bigg)^T f(d_{(j \to \uk)}) &= \bigg( g(c_{(i \to 1)}) -  g(c_{(i \to 0)}) \bigg)^T f(d_{(j \to c_j)}) \\
    \bigg( g(c_{(i \to 1)}) -  g(c_{(i \to 0)}) \bigg)^T & \bigg(f(d_{(j \to \uk)}) - f(d_{(j \to c_j)} \bigg) = 0
\end{split}
\end{equation*}

\end{proof}

\corortho*
\begin{proof}
Consider two binary vectors $c^{(0)}, c^{(1)} \in \mathcal{C}$ where $c^{(0)}_j = 0$ and $c^{(1)}_j = 1$ but they agree on other entries. 

By \cref{thm:ortho},
\begin{equation*}
    g(c^{(0)}_{(i \to 1)}) -  g(c^{(0)}_{(i \to 0)}) \perp f(d^{(0)}_{(j \to c_j)}) - f(d^{(0)}_{(j \to \uk)})
\end{equation*}
for any $d^{(0)} \in \mathcal{D}_{c^{(0)}}$. Similar statements can be made for $c^{(1)}$ as well.

Note that $\mathcal{D}_{c^{(0)}} \cap \mathcal{D}_{c^{(1)}} \neq \emptyset$ by the positivity assumption. Let $d \in \mathcal{D}_{c^{(0)}} \cap \mathcal{D}_{c^{(1)}}$. Then,
\begin{equation*}
\begin{split}
    g(c^{(0)}_{(i \to 1)}) -  g(c^{(0)}_{(i \to 0)}) &\perp f(d_{(j \to 0)}) - f(d_{(j \to \uk)}) \\
    g(c^{(1)}_{(i \to 1)}) -  g(c^{(1)}_{(i \to 0)}) &\perp f(d_{(j \to 1)}) - f(d_{(j \to \uk)}) \\
\end{split}
\end{equation*}
By assumption, there exists a unit vector $u_i$, such that
\begin{equation*}
    g(c^{(0)}_{(i \to 1)}) -  g(c^{(0)}_{(i \to 0)}) = \alpha^{(0)}u_i \quad g(c^{(1)}_{(i \to 1)}) -  g(c^{(1)}_{(i \to 0)}) = \alpha^{(1)}u_i
\end{equation*}
for some $\alpha^{(0)}, \alpha^{(1)} > 0$.
Therefore, 
\begin{equation*}
    \bigg\langle u_i, f(d_{(j \to 0)}) - f(d_{(j \to \uk)}) \bigg\rangle =  \bigg\langle u_i, f(d_{(j \to 1)}) - f(d_{(j \to \uk)}) \bigg\rangle = 0
\end{equation*}
Thus,
\begin{equation*}
    \bigg\langle u_i, f(d_{(j \to 0)}) - f(d_{(j \to 1)}) \bigg\rangle = 0
\end{equation*}
Because latent variable $C_i$ has linaer and matched representations, 

\begin{equation*}
    g(c_{(i \to 1)}) -  g(c_{(i \to 0)}) \perp g(c_{(j \to 1)}) -  g(c_{(j \to 0)})
\end{equation*}
for any $c \in \mathcal{C}$.
\end{proof}

\section{Simulated Experiments}
\label{app: sim}

In this section, we will provide additional details on our simulated experiments with the latent conditional model.

\paragraph{Additional details}
To let models learn conditional distributions, we train them to make predictions using cross-entropy loss. To turn the aforementioned generated binary vectors into a prediction task, we randomly generate binary masks $\mu$ for each vector in the batch. And if $\mu_i = 1$, the $i$-th entry of the vector is left untouched, and if $\mu_i = 0$, then it is set to $-1$ which is a numerical representation of the token $\uk$. The model is trained to use masked vectors to predict the original vectors.

These sampled binary vectors and ternary vectors are mapped into one-hot encodings to avoid neural networks exploiting the inherent structures of these vectors. $f$ and $g$ are modeled as a linear function, which are essentially lookup tables. This construction is made without loss of generality.

\paragraph{Adam optimizer} Although the theory is presented with gradient descent, the empirical result is actually robust to the choice of optimizers. We repeat the previous experiments on the complete set of conditionals with Adam optimizer \cite{kingma2014adam} using a learning rate $0.001$ and observe similar linear representation patterns in \cref{tab:adam}.

\begin{figure}[h]
    \centering
        \includegraphics[width=.6\linewidth]{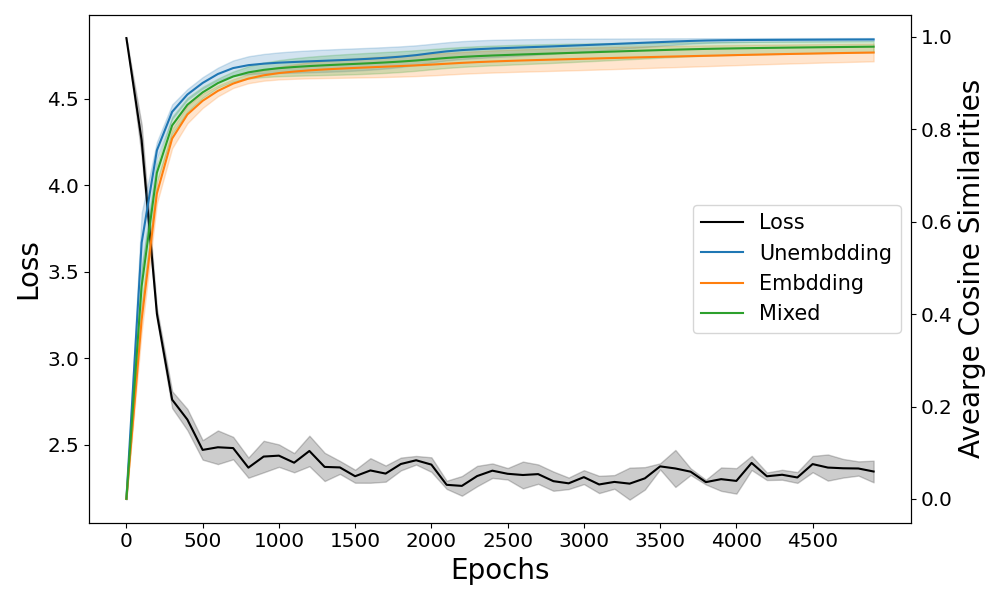}
        \caption{Loss and cosine similarities change as training progresses. The experiments are tested with $7$ hidden variables over 10 runs.}
        \label{fig:loss}
\end{figure}

\begin{table}[!h]
\caption{When the model is trained with Adam optimizer, the $m$ latent variables are represented linearly, and the embedding and unembedding representations are matched. The table shows average cosine similarities among and between steering vectors of unembeddings and embeddings. Standard errors are over $100$ runs for $3$ variables and $4$ variables, $50$ runs for $5$ variables, $20$ runs for $6$ variables and $10$ runs for $7$ variables.}
\label{tab:adam}
\begin{center}
\begin{small}
\begin{sc}
\begin{tabular}{lcccr}
\toprule
\makecell{$m$} & Unembedding & Embedding & \makecell{Unembedding \\ and Embedding}\\
\midrule
3    & 0.910$\pm$0.015 & 0.923$\pm$0.013 & 0.926$\pm$0.012 \\
4    & 0.972$\pm$0.005 & 0.959$\pm$0.005 & 0.965$\pm$0.005\\
5    & 0.985$\pm$0.004 & 0.970$\pm$0.004 & 0.975$\pm$0.004 \\
6    & 0.996$\pm$0.001 & 0.977$\pm$0.002 & 0.983$\pm$0.001 \\
7    & 0.966$\pm$0.012 & 0.918$\pm$0.016 & 0.929$\pm$0.015 \\
\bottomrule
\end{tabular}
\end{sc}
\end{small}
\end{center}
\end{table}

\paragraph{Incomplete set of contexts ($\widehat{\mathcal{D}} \subset \mathcal{D}$)} Both the size of $\mathcal{C}$ and $\mathcal{D}$ grow exponentially. To model large language models, not every conditional probability is necessarily trained. In fact, one does not need the complete set of contexts to get linearly encoded representations. To test this, we run experiments on incomplete subsets of contexts by randomly selecting a few masks $\{ \mu \}$. \cref{tab:incomplete-cond} shows that even with subsets of contexts, one can still get linearly encoded representations.

\begin{table}[h]
\caption{When the model is trained to learn subsets of conditionals for $10$ latent variables, the latent variables are still represented linearly and matched. The table shows average cosine similarities. Standard errors are over $10$ runs.}
\label{tab:incomplete-cond}
\begin{center}
\begin{small}
\begin{sc}
\begin{tabular}{lcccr}
\toprule
\makecell{Max Number \\ of Masks} & Unembedding & Embedding & \makecell{Unembedding \\ and Embedding}\\
\midrule
50    & 0.974$\pm$0.009 & 0.946$\pm$0.014 & 0.959$\pm$0.011 \\
100    & 0.957$\pm$0.009 & 0.915$\pm$0.013 & 0.934$\pm$0.011\\
\bottomrule
\end{tabular}
\end{sc}
\end{small}
\end{center}
\end{table}

\paragraph{Incomplete set of concept vectors or violation of positivity ($\widehat{\calC} \subset \calC$)}
In real-world language modeling, not every concept vector will be mapped to a token. We test this by allowing some concept vectors $c$ to have zero probabilities (i.e., $p(c) = 0$). For the experiments, we first randomly select a subset $\hat{\calC}$ of $\calC$ and then use rejection sampling to collect data points. \cref{tab:violation-pos} shows that one can still get reasonable linearity for unembeddings. The embedding alignments drop, however. One possible explanation is due to a lack of training because the size of the problem grows exponentially. On the other hand, \cref{sec: bias_linearity} suggests the connection between linearity and zero conditional probabilities. Because the positivity assumption is violated, the newly introduced zero conditional probabilities might also cause misalignment.

\begin{table}[h]
\caption{When training with an incomplete set of concept vectors and contexts, unembedding representations are still encoded linearly. The table shows average cosine similarities. Standard errors are over $10$ runs.}
\label{tab:violation-pos}
\begin{center}
\begin{small}
\begin{sc}
\begin{tabular}{lcccr}
\toprule
\makecell{Number of \\ Hidden Varibles} & Unembedding & Embedding & \makecell{Unembedding \\ and Embedding}\\
\midrule
10    & 0.951$\pm$0.011 & 0.777$\pm$0.010 & 0.855$\pm$0.011 \\
12    & 0.896$\pm$0.011 & 0.551$\pm$0.008 & 0.696$\pm$0.009 \\
\bottomrule
\end{tabular}
\end{sc}
\end{small}
\end{center}
\end{table}

\paragraph{Change of dimensions}
Previous experiments set the representation dimension to be the same as the number of latent variables. In this set of experiments, we test how decreasing dimensions affects representations by rerunning experiments on $7$ variables with the complete set of contexts and binary vectors as well as experiments with $10$ variables with incomplete set of contexts and binary vectors. \cref{fig:change-of-dimensions} shows that although decreasing dimensions do make representations less aligned, the effect is not significant. 

\begin{figure}[t]
    \centering
    \begin{subfigure}{0.49\textwidth}
        \includegraphics[width=\linewidth]{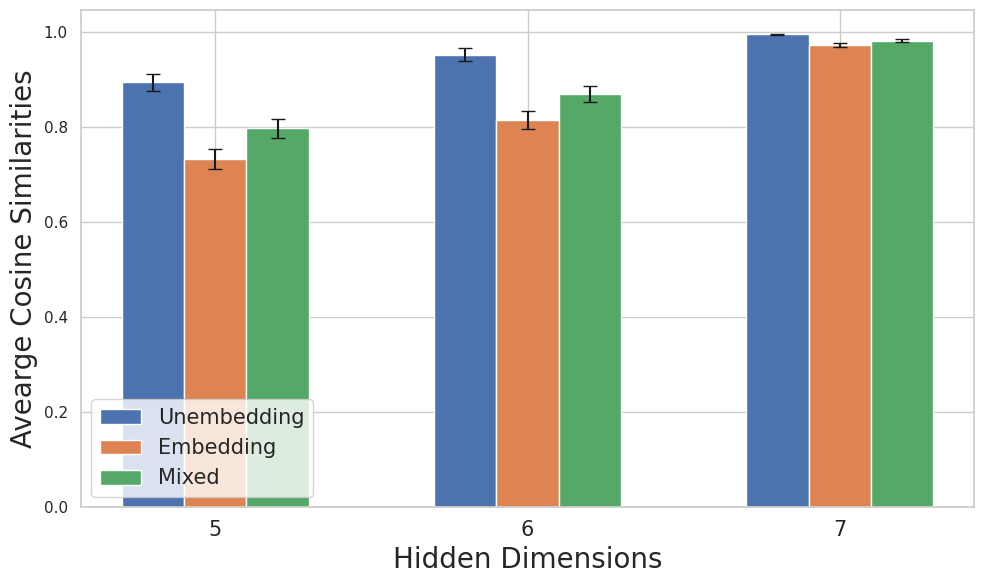}
        \caption{$7$ Hidden Variables}
        \label{fig:sub1}
    \end{subfigure}
    \hfill
    \begin{subfigure}{0.49\textwidth}
        \includegraphics[width=\linewidth]{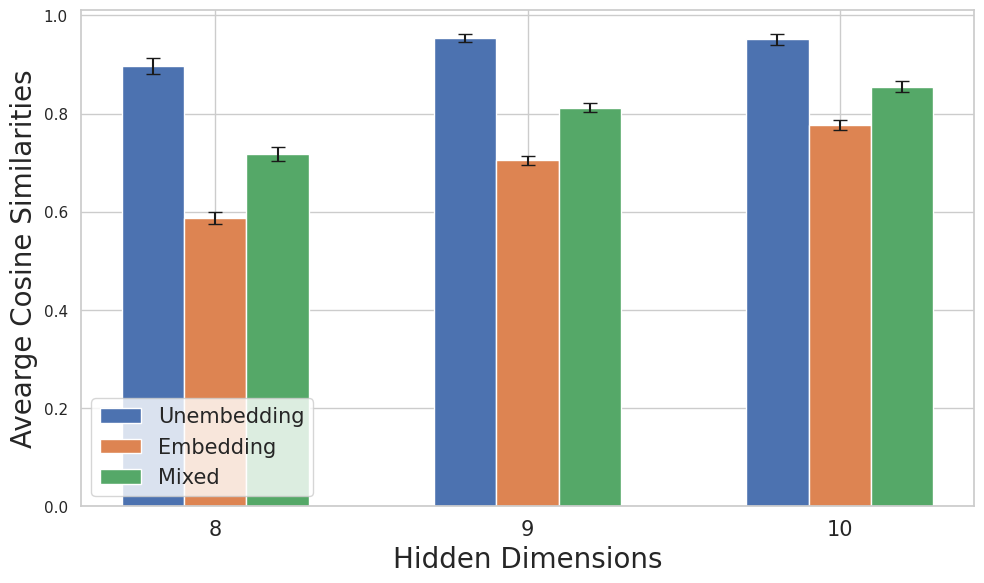}
        \caption{$10$ Hidden Variables}
        \label{fig:sub2}
    \end{subfigure}
    \caption{Average cosine similarities under different hidden dimensions show that reducing dimension dose not hurt linearity significantly.}
    \label{fig:change-of-dimensions}
\end{figure}

\begin{figure}[h]
    \centering
    \begin{subfigure}{0.49\textwidth}
        \includegraphics[width=\linewidth]{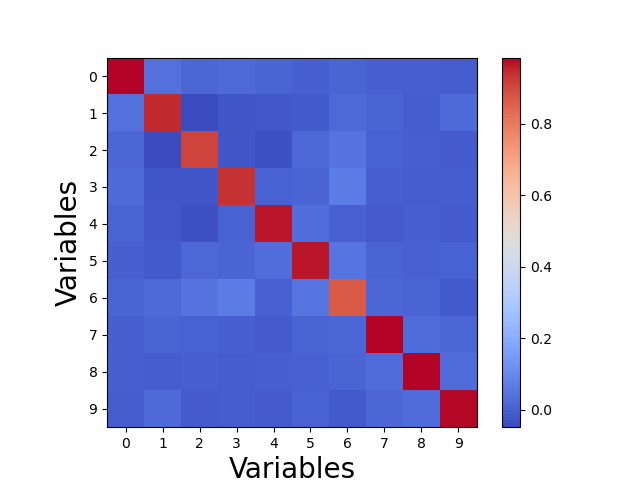}
        \caption{Simulated Experiments}
        \label{fig:sub1}
    \end{subfigure}
    \hfill
    \begin{subfigure}{0.49\textwidth}
        \includegraphics[width=\linewidth]{pics/heatmap_llm.png}
        \caption{LLaMA-2}
        \label{fig:sub2}
    \end{subfigure}
    \caption{The pattern of linear and orthogonal representations matches between simulated experiments and LLaMA-2. Specifically, unembedding steering vectors of the same concept in LLaMA-2 have nontrivial alignment, while steering vectors of different concepts are represented almost orthogonally. The heatmap shows the average cosine similarities between different sets of steering vectors for both simulated experiments and LLaMA-2. For simulated experiments, the cosine similarities are averaged over $10$ runs.}
    \label{fig:heatmap}
\end{figure}

\section{Experiments with large language models}
\label{app: llm_expts}

In this section, we will provide additional details on our experiments with LLMs.

\paragraph{Examples of counterfactual context and token pairs}
Example context pairs
from OPUS Books \cite{tiedemann2012parallel}
that were used to construct the embedding vectors are shown in \cref{tab: example_context_pairs_lang}.
For the unembedding vectors, we reuse the 27 concepts considered in \cite{park2023linear} built atop the work of \cite{gladkova2016analogy}, and they are listed in \cref{tab: 27_concepts} along with example token pairs.

\paragraph{Winograd Schema}
We now present full details about our experiments on the Winograd Schema dataset \cite{levesque2012winograd}.
Recall that we consider context pairs arising from the Winograd Schema, which is a dataset of pairs of sentences that differ in only one or two words and which further contain an ambiguity that can only be resolved with world knowledge and reasoning.
Example context pairs from the Winograd Schema \cite{levesque2012winograd} are shown in \cref{tab: example_context_pairs_wino}.
Note the final 2 context pairs in the table, which have ambiguous concepts, thus highlighting the nuances of this dataset as well as of natural language.
We filter the dataset to have the ambiguous word near the end of the sentence to better align with our theory.

We again compute the unembedding vectors for these context pairs output by LLaMA-2. For the embedding vectors, since there is no predefined set of concepts to consider (for instance consider the ambiguous examples), we therefore take the difference of the embedding vectors for the first token that differs in these corresponding pairs of sentences.
We compute all pairwise cosine similarities. We observe that for non-matching contexts, the average similarity is 0.011 with a maximum of 0.081, whereas for matching contexts, the average similarity is 0.042 with a maximum of 0.161. This aligns with our predictions that the embedding vectors align better with the unembedding vectors of the same concept. 

As an additional experiment, we compute similarities between these Winograd context pairs and the 27 concepts from \cite{park2023linear} described above. In \cref{tab: top_similarities}, we display the top 3 pairs of contexts and concepts that have the highest similarities. The alignment seems reasonable as baby, woman, male, and female are different attributes of a person; and wide vs narrow and short vs tall can be construed as different manifestations of small vs big.

\begin{table}[!h]
\caption{Example language context pairs from OPUS Books}
\label{tab: example_context_pairs_lang}
\begin{center}
\begin{footnotesize}
\begin{tabular}{ccc}
\toprule
\sc{Language pair} & \sc{Context 1} & \sc{Context 2}\\
\midrule
French--Spanish & Quinze ou seize, r\'{e}pliqua l'autre & Quince \'{o} diez y seis, replic\'{o} el otro\\
French--German & Comment est-il mon ma\^{i}tre? & Wie ist er mein Herr?\\
English--French & I hesitated for a moment. & J’h\'{e}sitai une seconde.\\
German--Spanish & Ich hasse die Spazierfahrten & No me gusta salir en coche.\\
\bottomrule
\end{tabular}
\end{footnotesize}
\end{center}
\end{table}

\begin{table}[!h]
    \centering
    \caption{Concepts and example token pairs, taken from \cite{park2023linear}}
    \label{tab: 27_concepts}
{\footnotesize
    \begin{tabular}{cc|cc}
        \toprule
        \sc{Concept} & \sc{Example token pair} & \sc{Concept} & \sc{Example token pair} \\
        \midrule
        verb $\Rightarrow$ 3pSg & (accept, accepts) & verb $\Rightarrow$ Ving & (add, adding) \\
        verb $\Rightarrow$ Ved & (accept, accepted) & Ving $\Rightarrow$ 3pSg & (adding, adds)\\
        Ving $\Rightarrow$ Ved & (adding, added) & 3pSg $\Rightarrow$ Ved & (adds, added) \\
        verb $\Rightarrow$ V + able & (accept, acceptable) & verb $\Rightarrow$ V + er & (begin, beginner)\\
        verb $\Rightarrow$ V + tion & (compile, compilation) & verb $\Rightarrow$ V + ment & (agree, agreement)\\
        adj $\Rightarrow$ un + adj & (able, unable) & adj $\Rightarrow$ adj + ly & (according, accordingly)\\
        small $\Rightarrow$ big & (brief, long) & thing $\Rightarrow$ color & (ant, black)\\
        thing $\Rightarrow$ part & (bus, seats) & country $\Rightarrow$ capital & (Austria, Vienna) \\
        pronoun $\Rightarrow$ possessive & (he, his) & male $\Rightarrow$ female & (actor, actress) \\
        lower $\Rightarrow$ upper & (always, Always) & noun $\Rightarrow$ plural & (album, albums) \\
        adj $\Rightarrow$ comparative & (bad, worse) & adj $\Rightarrow$ superlative & (bad, worst)\\
        frequent $\Rightarrow$ infrequent & (bad, terrible) & English $\Rightarrow$ French & (April, avril)\\
        French $\Rightarrow$ German & (ami, Freund) & French $\Rightarrow$ Spanish & (année, año)\\
        German $\Rightarrow$ Spanish & (Arbeit, trabajo)\\
        \bottomrule
    \end{tabular}}
\end{table}

\begin{table}[!h]
\caption{Example context pairs from Winograd Schema}
\label{tab: example_context_pairs_wino}
\begin{center}
\begin{footnotesize}
\begin{tabular}{c}
\toprule
\sc{Contexts}\\
\midrule
The delivery truck zoomed by the school bus because it was going so fast.\\ The delivery truck zoomed by the school bus because it was going so slow.\\
\midrule
The man couldn't lift his son because he was so weak.\\
The man couldn't lift his son because he was so heavy.\\
\midrule
Joe's uncle can still beat him at tennis, even though he is 30 years younger.\\
Joe's uncle can still beat him at tennis, even though he is 30 years older.\\
\midrule
Paul tried to call George on the phone, but he wasn't successful.\\
Paul tried to call George on the phone, but he wasn't available.\\
\midrule
The large ball crashed right through the table because it was made of steel.\\
The large ball crashed right through the table because it was made of styrofoam.\\
\bottomrule
\end{tabular}
\end{footnotesize}
\end{center}
\end{table}

\begin{table}[h]
\caption{Top similarities between Winograd contexts and token concepts}
\label{tab: top_similarities}
\begin{center}
\begin{footnotesize}
\begin{tabular}{lcc}
\toprule
\sc{Contexts} & \sc{Most similar concept} & \sc{Similarity}\\
\midrule
Anne gave birth to a daughter last month. She is a very charming woman.
& \multirow{2}{*}{male$\Rightarrow$female} & \multirow{2}{*}{0.311}\\
Anne gave birth to a daughter last month. She is a very charming baby .&&\\
\midrule
The table won't fit through the doorway because it is too wide. & \multirow{2}{*}{small$\Rightarrow$big} & \multirow{2}{*}{0.309}\\
The table won't fit through the doorway because it is too narrow.&&\\
\midrule
John couldn't see the stage with Billy in front of him because he is so short.&\multirow{2}{*}{small$\Rightarrow$big} & \multirow{2}{*}{0.303}\\
John couldn't see the stage with Billy in front of him because he is so tall.&&\\
\bottomrule
\end{tabular}
\end{footnotesize}
\end{center}
\end{table}

\paragraph{Additional barplots}
The entire set of similarity barplots for the concepts of French--Spanish, French--German, English--French and German--Spanish are in Figures \ref{fig: es-fr}, \ref{fig: de-fr}, \ref{fig: en-fr} and \ref{fig: de-es} respectively.
As we see, they satisfy the same behavior as described earlier in \cref{sec: llm_expts}, exhibiting relatively high similarity with the matching unembedding vector, close to high similarity with related language concepts and low similarity with unrelated concepts.

\begin{figure}[!h]
\centering
\includegraphics[scale=0.25, trim={0em 0em 0em 0em},clip]{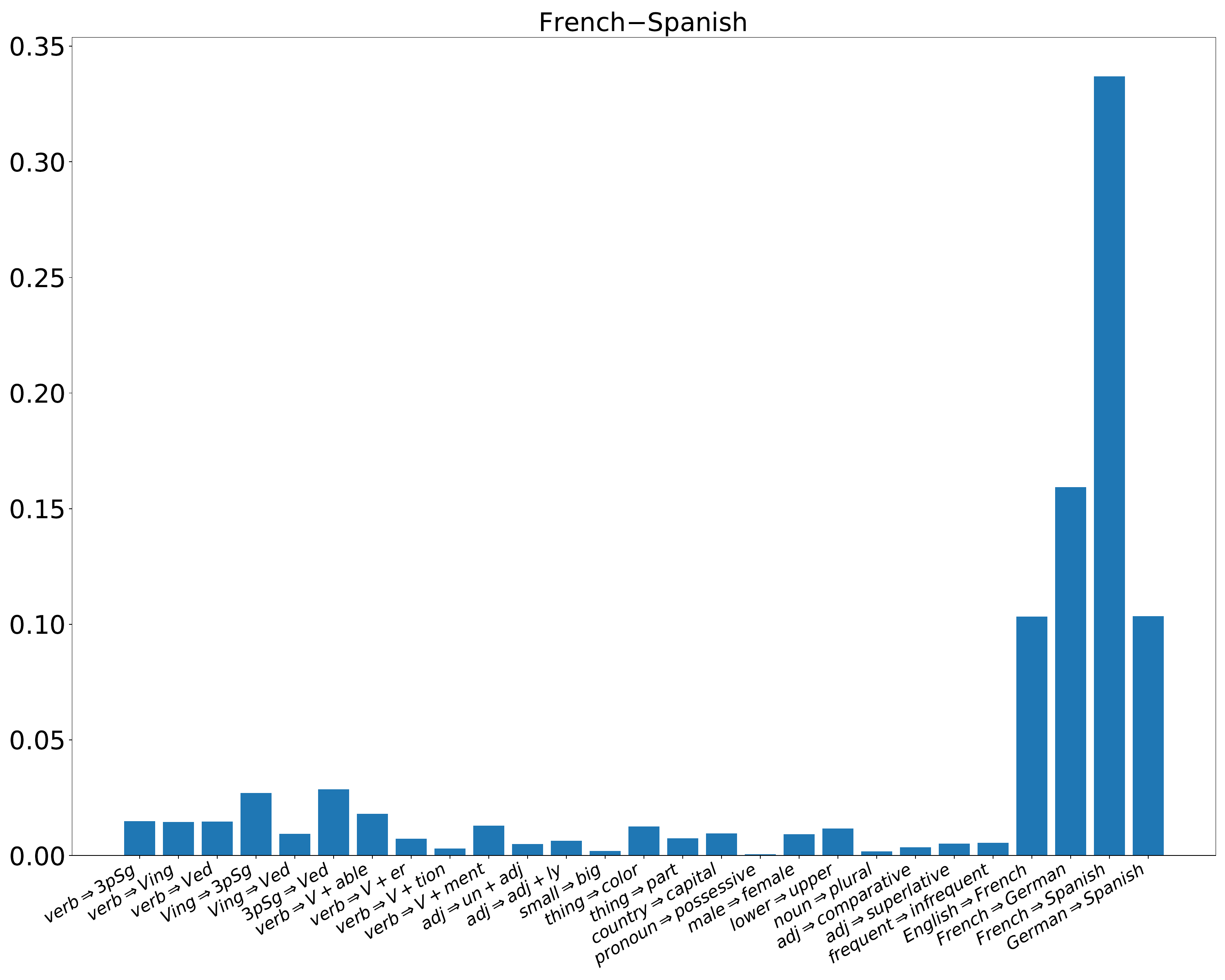}
\caption{The French--Spanish concept is highly correlated with similar token concepts relative to others. This figure shows all cosine similarities between the French--Spanish concept and token concepts.}
\label{fig: es-fr}
\end{figure}

\begin{figure}[!h]
\centering
\includegraphics[scale=0.25, trim={0em 0em 0em 0em},clip]{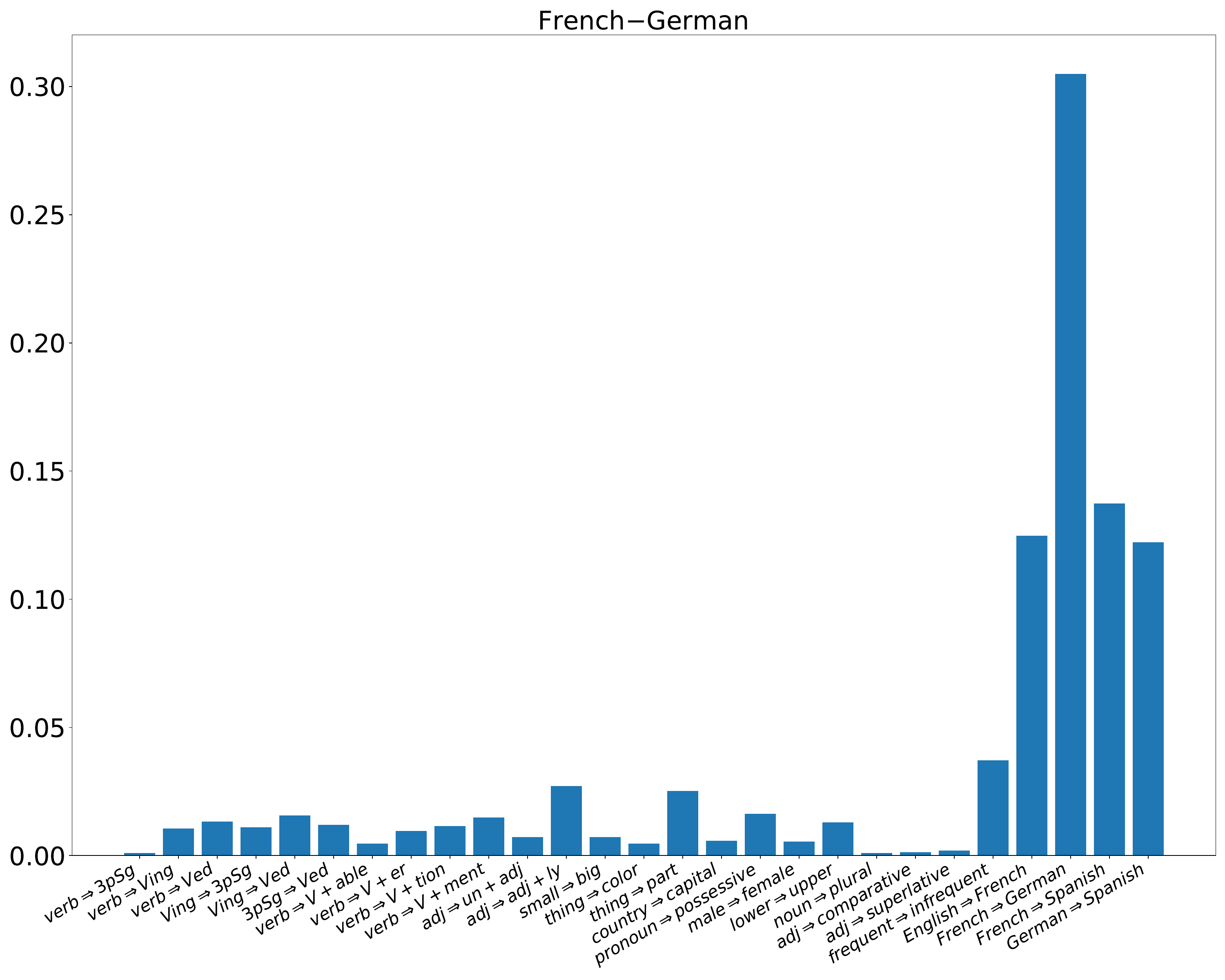}
\caption{The French--German concept is highly correlated with similar token concepts relative to others. This figure shows all cosine similarities between the French--German concept and token concepts.}
\label{fig: de-fr}
\end{figure}

\begin{figure}[!h]
\centering
\includegraphics[scale=0.25, trim={0em 0em 0em 0em},clip]{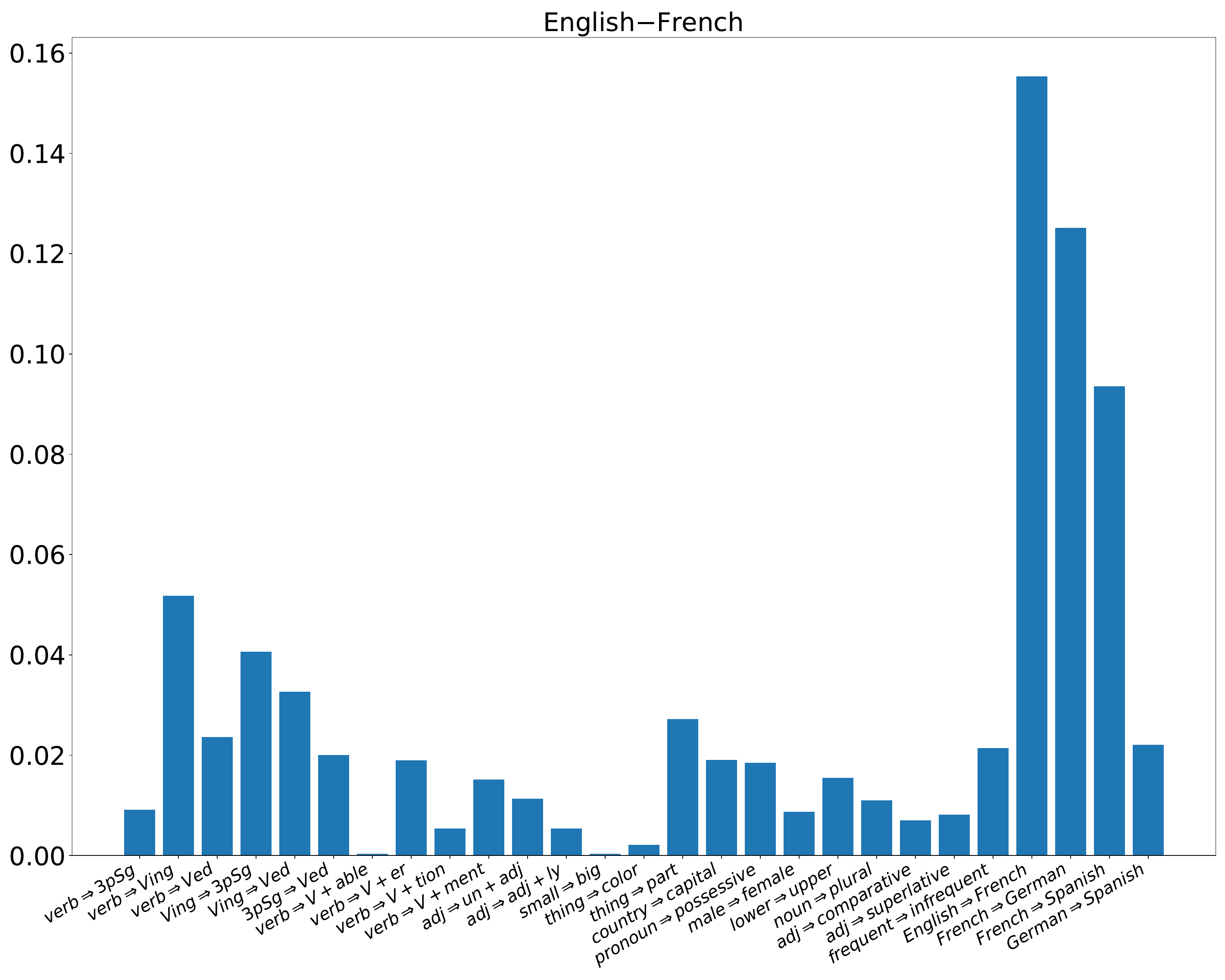}
\caption{The English--French concept is highly correlated with similar token concepts relative to others. This figure shows all cosine similarities between the English--French concept and token concepts.}
\label{fig: en-fr}
\end{figure}

\begin{figure}[!h]
\centering
\includegraphics[scale=0.25, trim={0em 0em 0em 0em},clip]{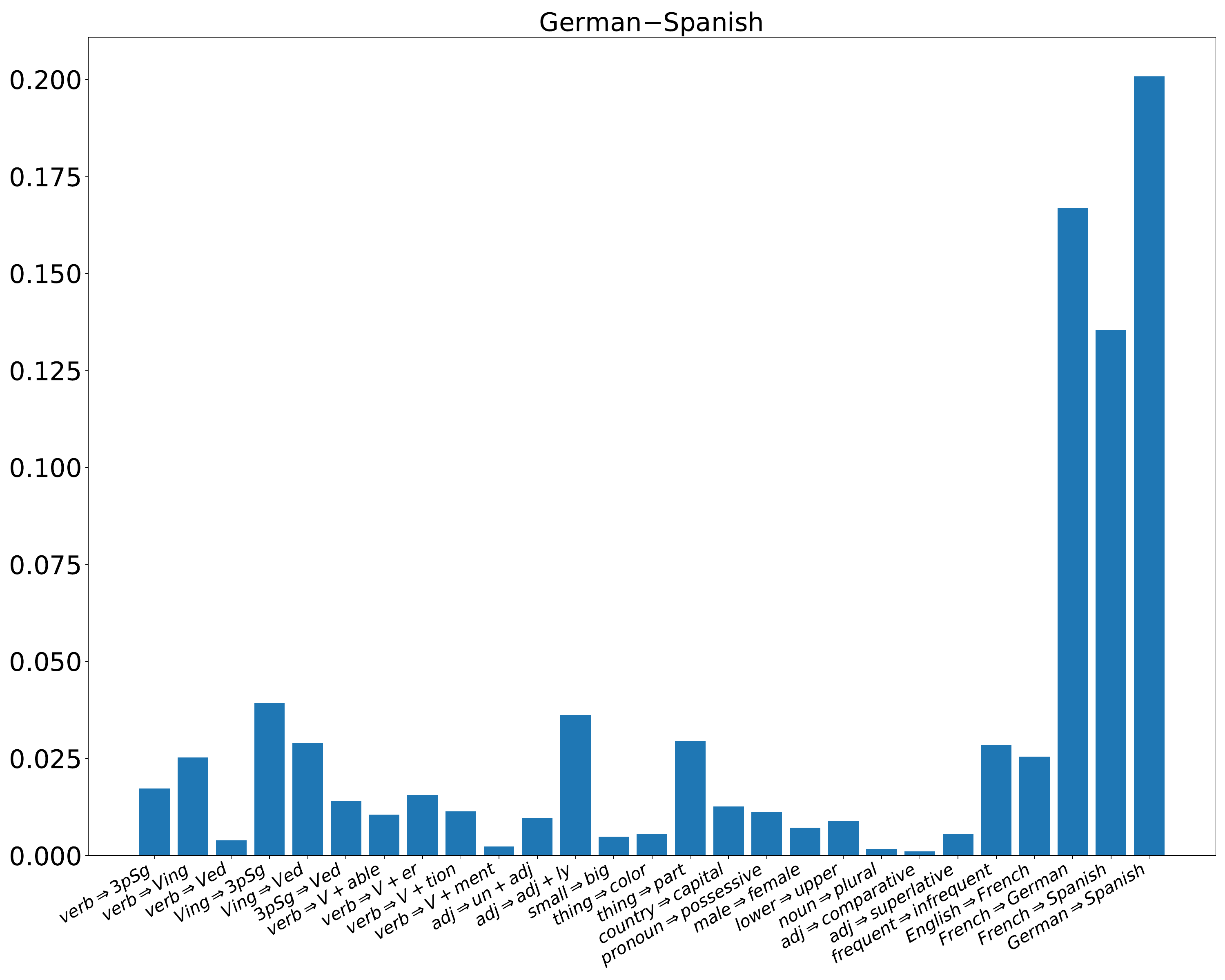}
\caption{The German--Spanish concept is highly correlated with similar token concepts relative to others. This figure shows all cosine similarities between the German--Spanish concept and token concepts.}
\label{fig: de-es}
\end{figure}


\end{document}